\documentclass[twoside]{article}
\usepackage[accepted]{aistats2022}
\usepackage[round]{natbib}

\usepackage{amsmath}
\usepackage[utf8]{inputenc} % allow utf-8 input
\usepackage[T1]{fontenc}    % use 8-bit T1 fonts
\usepackage{hyperref}       % hyperlinks
\usepackage{url}            % simple URL typesetting
\usepackage{booktabs}       % professional-quality tables
\usepackage{amsfonts}       % blackboard math symbols
\usepackage{nicefrac}       % compact symbols for 1/2, etc.
\usepackage{microtype}      % microtypography
\usepackage{xcolor}         % colors
\usepackage{amsmath}
\usepackage{float}
\usepackage{graphicx}
\usepackage{subfig}
\usepackage{amssymb}
\usepackage{algorithm,algorithmic}
\usepackage{amsthm}
\newtheorem{theorem}{Theorem}[section]
\newtheorem{lemma}[theorem]{Lemma}
\newtheorem{proposition}[theorem]{Proposition}

\newtheorem{innercustomgeneric}{\customgenericname}
\providecommand{\customgenericname}{}
\newcommand{\newcustomtheorem}[2]{%
  \newenvironment{#1}[1]
  {%
  \renewcommand\customgenericname{#2}%
  \renewcommand\theinnercustomgeneric{##1}%
  \innercustomgeneric
  }
  {\endinnercustomgeneric}
}
\newcustomtheorem{customthm}{Theorem}
\newcustomtheorem{customlemma}{Lemma}

\newtheorem{assumption}{Assumption}
\begin{document}



\twocolumn[

\aistatstitle{Schedule Based Temporal Difference Algorithms}

\aistatsauthor{Rohan Deb$^{*1}$ \And Meet Gandhi$^{*1}$ \And Shalabh Bhatnagar$^{1}$ }

\aistatsaddress{ $^{1}$Department of Computer Science and Automation\\
Indian Institute of Science, \\
Bangalore\\
\{rohandeb,meetgandhi,shalabh\}@iisc.ac.in} ]
\begin{abstract}
Learning the value function of a given policy from data samples is an important problem in Reinforcement Learning.
TD($\lambda$) is a popular class of algorithms to solve this problem.
However, the weights assigned to different $n$-step returns in TD($\lambda$), controlled by the parameter $\lambda$, decrease exponentially with increasing $n$.
In this paper, we present a $\lambda$-schedule procedure that generalizes the TD($\lambda$) algorithm to the case when the parameter $\lambda$ could vary with time-step. This allows flexibility in weight assignment, i.e., the user can specify the weights assigned to different $n$-step returns by choosing a sequence $\{\lambda_t\}_{t \geq 1}$.
Based on this procedure, we propose an on-policy algorithm -- TD($\lambda$)-schedule, and two off-policy algorithms -- GTD($\lambda$)-schedule and TDC($\lambda$)-schedule,
respectively. We provide proofs of almost sure convergence for all three algorithms under a general Markov noise framework.
\end{abstract}

\section{Introduction}
\label{sec_introduction}

Reinforcement Learning (RL) problems can be categorised into two classes: prediction and control. 
The prediction problem deals with estimating the value function of a given policy as accurately as possible. 
Obtaining precise estimates of the value function is an important problem because value function provides useful information, such as, importance of being in different game positions in Atari games (\cite{DQN}), taxi-out times at big airports (\cite{TaxioutTime}), failure probability in large communication networks (\cite{FailureProbTelNetworks}), etc.
See \cite{Dann} for a discussion on the prediction problem.

Evaluating the value function is an easy task when the state space is finite and the model of the system (transition probability matrix and single-stage reward function) is known. 
However, in many practical scenarios, the state space is large and the transition probability kernel is not available. 
Instead, samples in the form of (state, action, reward, next-state) are available and the value function needs to be estimated from these samples.
In the RL community, learning using the samples generated from the actual or simulated interaction with the environment is called model-free learning.

Monte-Carlo (MC) methods and one-step Temporal Difference (TD) methods are popular algorithms for estimating the value function in the model-free setting. 
The $n$-step TD is a generalization of one-step TD, wherein the TD error is obtained as the difference between the current estimate of the value function and the $n$-step return, instead of the one-step return. 
These $n$-step methods span a spectrum with MC at one end and one-step TD at the other. 
The TD($\lambda$) algorithm takes the next logical step of combining the $n$-step returns for different values of $n$.
A single parameter $\lambda$ decides the weight assigned to different $n$-step returns, which decreases exponentially as $n$ increases.

\subsection{Motivation and Contributions}
In some situations, it is observed that $n$-step TD for some intermediate values of $n$ outperforms the same for both small as well as large values of $n$ (cf. Figure 7.2 of \cite{SuttonBarto_book}). It's demonstrated empirically in
\cite{SuttonBarto_book} that neither MC nor one-step TD performs the best in terms of minimizing the Mean-Squared Error (MSE),
% Some in-between value of n works the best. 
because one-step TD is highly biased whereas MC has high variance.
As a result, intermediate values of $n$ are likely to achieve the lowest MSE.
With the TD($\lambda$) algorithm, we cannot combine only specifically chosen $n$ step returns such as one described above, as it assigns a weight of $(1-\lambda)\lambda^n$ to the $n$-step return for each $n$.
In this paper, we design a $\lambda$-schedule procedure that allows flexibility in weight assignment to the different $n$-step returns.
Specifically, we generalize the TD($\lambda$) algorithm to the case where the parameter $\lambda$ depends on the time-step. This produces a sequence $\{\lambda_t\}_{t\geq1}$.
Using this procedure, we develop an on-policy algorithm called TD($\lambda$)-schedule and two off-policy algorithms called GTD($\lambda$)-schedule and TDC($\lambda$)-schedule. 
We prove the convergence of all the three algorithms under a general Markov noise framework. 
Even though we consider the state space to be finite for ease of exposition, our proofs carry through easily to the case of general state spaces under additional assumptions.
We point out here that while the TD($\lambda$)-schedule and GTD($\lambda$)-schedule algorithms are single-timescale algorithms with Markov noise, the TDC($\lambda$)-schedule algorithm in fact involves two timescales with Markov noise in both the slower and faster iterates.
Our proof techniques are more general as compared to others in the literature. For instance, \cite{TsitsiklisVanRoy} prove the convergence of TD($\lambda$) using the results of \cite{Benveniste_book}. 
However, to the best of our knowledge, there are no known generalizations of that result to the case of two-timescale algorithms such as TDC($\lambda$)-schedule. 
Moreover, unlike \cite{TsitsiklisVanRoy}, our results can further be extended to the case where the underlying Markov chain does not possess a unique stationary distribution but a set of such distributions that could even depend on an additional control process.
%Finally, we present the results of experiments where we observe a clear improvement in performance.
% In this paper, we propose an algorithm where the user can assign different weights to different $n$-step returns (under certain mild conditions). 
% We call this \textit{TD($\lambda$)-schedule} algorithm. 
% We prove that the proposed algorithm converges with probability one for the on-policy case under Markov noise setting. 
% Subsequently, we examine TD($\lambda$)-schedule on Baird's counterexample (an off-policy prediction problem) and observe that the algorithm diverges (See Supplementary material).
% Hence, inspired by the gradient-based algorithms (\cite{Maei_PhD}), we propose two off-policy algorithms, called \textit{GTD($\lambda$)-schedule} and \textit{TDC($\lambda$)-schedule}. 
% We prove that both these algorithms converge with probability one for the off-policy case under Markov noise setting. 
% We highlight that the proof technique used by Tsitsiklis and Van Roy does not directly extend to the case of two-timescale stochastic approximation, since to the best of our knowledge there are no extensions of the result by \cite{Benveniste_book} on stochastic approximation with Markov noise to the setting of two-timescale stochastic approximation algorithms. 
% Here, we have presented a novel almost sure convergence proof of TDC($\lambda$)-schedule, that works for two-timescale stochastic approximation algorithms under the Markov noise setting.
\section{\texorpdfstring{On-policy TD(\(\lambda\))}--schedule}
\label{sec_motivation}
In this section, we precisely define the on-policy TD($\lambda$)-schedule algorithm for an infinite-horizon discounted reward Markov chain induced by the deterministic policy $\pi$.
We note that though our results are applicable to Markov chains with general state space, we restrict our attention to the case where the state space is finite. 
Thus, the Markov chain can be defined in terms of a transition probability matrix as opposed to a transition probability kernel. 
%The extension to the case of general state-space is straightforward.

We assume that the Markov chain induced by the fixed policy $\pi$ is irreducible and aperiodic, whose states lie in a discrete state space $\mathcal{S}$. 
We can index the state space with positive integers, and view $\mathcal{S} = \{1,2,\ldots, n\}$. 
Each state $s \in \mathcal{S}$ has a corresponding feature vector $\phi(s) \in \mathcal{R}^d$ associated with it.
We denote $\{s_t | t=0,1,\ldots\}$ as the sequence of states visited by the Markov chain. 
The transition probability matrix $P$ induced by the Markov chain has $(i, j)^{th}$ entry, denoted by $p_{ij}$, as probability of going from state $s_t = i$ to $s_{t+1} = j$. 
% Also, the scalar $R_{t+1}$ represents the single-stage reward obtained when the system transitions from state $s_t$ to state $s_{t+1}$. 
Also, the scalar $R_{t+1} \equiv r(s_t, s_{t+1})$ represents the single-stage reward obtained when the system transitions from state $s_t$ to state $s_{t+1}$.
Since the state space is finite, $\sup_t |R_{t+1}| <\infty$ almost surely.
We let $\gamma \in (0,1)$ be the discount factor.
The value function $V^{\pi}: \mathcal{S} \rightarrow \mathcal{R}$ associated with this Markov chain is given by
\(V^{\pi}(s) = \mathbb{E}\Big[\sum_{t=0}^{\infty}\gamma^t R_{t+1}| s_0 = s\Big].\)
% \begin{equation}
%     V^{\pi}(s) = \mathbb{E}\Big[\sum_{t=0}^{\infty}\gamma^t R_{t+1}| s_0 = s\Big].
% \end{equation}
The above expectation is well-defined because the single-stage reward is bounded as mentioned above.
We consider approximations of $V^{\pi}: \mathcal{S} \rightarrow R$ using function $V_{\theta}: \mathcal{S} \times \mathcal{R}^d \rightarrow \mathcal{R}$, where $V_{\theta}$ is a linear function approximator parameterized by $\theta$, i.e., $V_{\theta}(s) = \theta^{T}\phi(s)$. 
Our aim is to find the parameter $\theta \in \mathcal{R}^d$ to minimise the \emph{Mean Squared Error (MSE)} between the true value function $V^{\pi}(\cdot)$ and the approximated value function $V_{\theta}(\cdot)$ for a given $\theta$, where
\begin{equation}
    MSE (\theta) = \sum_{s \in \mathcal{S}} d^{\pi}(s)\left[V^{\pi}(s) - V_{\theta}(s)\right]^2 = ||V^{\pi} - V_{\theta}||^2_{D}.
\end{equation} 
Here, $\{d^{\pi}(s)\}_{s \in \mathcal{S}}$ represents the steady-state distribution for the Markov chain and the matrix $D$ is a diagonal matrix of dimension $n \times n$ with the entries $d^{\pi}(s)$ on its diagonals.
Minimising MSE with respect to $\theta$ by stochastic-gradient descent gives the update equation for $\theta$ as, 
\begin{equation}
    \theta_{t+1} = \theta_{t} + \alpha_t [V^{\pi}(s_t) - V_{\theta}(s_t)]\phi(s_t),
\end{equation}
where $\theta_t$ is the value of parameter $\theta$ at time $t$. 
We now motivate the main idea of the paper. 
We propose an algorithm where the user can assign the weights to different $n$-step returns to estimate $V^{\pi}(s_t)$.
We use the discounted-aware setting as described in Section-5.8 of  \cite{SuttonBarto_book} to define the return, which is used as an estimate of $V^{\pi}(s_t)$.
\begin{equation}
\label{lambda_return}
\begin{split}
    G_{t}^{\Lambda(\cdot)}(\theta) & = (1-\gamma)[\Lambda_{11}R_{t+1}] \\
 & + \gamma(1-\gamma)[\Lambda_{21}(R_{t+1}+V_{\theta}(s_{t+1}))
 \\ & \qquad\qquad +\Lambda_{22}(R_{t+1}+R_{t+2})] \\
 & + \gamma^2(1-\gamma)[\Lambda_{31}(R_{t+1} + V_{\theta}(s_{t+1})) 
 \\ & \qquad\qquad\quad + \Lambda_{32}(R_{t+1}+R_{t+2}+V_{\theta}(s_{t+2})) \\ & \qquad \qquad\quad +\Lambda_{33}(R_{t+1}+R_{t+2}+R_{t+3})] +\cdots
\end{split}
\end{equation}
The above equation is interpreted as follows: the episode ends in one step with probability $(1-\gamma)$, in two steps with probability $\gamma(1-\gamma)$, in three steps with probability $\gamma^{2}(1-\gamma)$ and so on. 
When the episode ends in one step, bootstrapping is not applicable and thus the flat return $R_{t+1}$ is weighed by $\Lambda_{11} = 1$. 
When the episode ends in two steps, the following two choices are available: bootstrap after one step ($R_{t+1} + V_{\theta}(s_{t+1})$) or use the flat return ($R_{t+1} + R_{t+2}$).
%Here, the episode ends in just one step with the probability of $1-\gamma$, in which case, $R_{t+1}$ is the return and $\Lambda_{11}$, which has to be $1$, is the corresponding weight.
%Similarly, the episode ends in two steps with the probability of $\gamma(1-\gamma)$, for which we have two choices: take the monte-carlo return for the trajectory of length two, which gives the return as $R_{t+1} + R_{t+2}$, or take one step return with bootstrap, which gives the return as $R_{t+1} + \hat{V}(s_{t+1})$. 
We weight these two quantities by $\Lambda_{21}$ and $\Lambda_{22}$ respectively under the constraint that $\Lambda_{21}$ and $\Lambda_{22}$ are non-negative and sum to 1.
Similarly, when the episode ends in $3$ steps, we have three choices: bootstrap after one step, bootstrap after two steps or take the flat return. We weight these three quantities by $\Lambda_{31}$, $\Lambda_{32}$ and $\Lambda_{33}$ respectively, under the constraint that these three weights are non-negative and sum to one, and so on. 
These weights can be summarized in a matrix as below, where each $\Lambda_{ij} \in [0,1]$ and weights in each row sum to one.% thereby resulting in a stochastic matrix. 
\begin{center}
    \[\Lambda = 
    \begin{bmatrix}
        1 & 0 & 0 & 0 \cdots\\
        \Lambda_{21} & \Lambda_{22} & 0 & 0 \cdots \\
        \Lambda_{31} & \Lambda_{32} & \Lambda_{33} & 0 \cdots \\
        
        \vdots & \vdots & \vdots &\ddots
    \end{bmatrix}.\]
 \end{center}
To obtain an online equation for $G_{t}^{\Lambda(\cdot)}(\theta) - V_{\theta}(s_{t})$, we add and subtract $V_{\theta}(s_{t+1})$ to all the terms starting from $R_{t+1}+R_{t+2}$ in \eqref{lambda_return}.
We notice that, on RHS, the coefficient of $R_{t+1}$ is 1. Similarly, the coefficient of $V_{\theta}(s_{t+1})$ is $\gamma$ (See Appendix A1 for details). Hence, 
\begin{equation*}
\begin{split}
    G_{t}^{\Lambda(\cdot)}(\theta) &- V_{\theta}(s_t)  = R_{t+1} + \gamma V_{\theta}(s_{t+1}) - V_{\theta}(s_{t}) \\
    & + \gamma\Lambda_{22}\{ (1-\gamma)(R_{t+2}) \\ &\qquad+ \gamma(1-\gamma)\Big[\frac{\Lambda_{32}}{\Lambda_{22}}(R_{t+2}+V_{\theta}(s_{t+2})) \\ & \qquad\qquad + \frac{\Lambda_{33}}{\Lambda_{22}}(R_{t+2}+R_{t+3})\Big] \\
    & \qquad + \cdots - V_{\theta}(s_{t+1})  \}.
\end{split}
\end{equation*}
To write the above equation recursively, we notice that we need to have the additional constraint $\Lambda_{32} + \Lambda_{33} = \Lambda_{22}$. 
In general we must ensure that $\Lambda_{j,j-1} + \Lambda_{j,j} = \Lambda_{j-1, j-1}$ $\forall j \geq 2$. Setting $\Lambda_{22} = \lambda_{1}$ and using the above constraint, we obtain $\Lambda_{21} = (1-\lambda_{1}).$ 
Further, setting $\Lambda_{33} = \lambda_{1}\lambda_{2}$, $\Lambda_{32} = \lambda_{1}(1-\lambda_{2})$ (to ensure that $\Lambda_{32} + \Lambda_{33} = \Lambda_{22}$), we obtain $\Lambda_{31} = (1-\lambda_{1})$. 
We refer to the user-specified sequence $\lambda_{j}, j \in \mathbb{N}$ as the \emph{$\lambda$-schedule} hereafter.
The weight matrix $\Lambda$ can be constructed from the user specified $\lambda$-schedule as below:
%Denoting $\Lambda_{22}$ as $\lambda_1$, $\Lambda_{21}$ becomes $1 - \lambda_1$.
%Denoting $\Lambda_{33}$ as $\lambda_1\lambda_2$ and  $\Lambda_{32}$ as $(1-\lambda_2)\lambda_1$ (which sum to $\Lambda_{22} = \lambda_1$), the weight matrix turns out to be as below.
%In fact, we need constraints that $\Lambda_{j,j-1} + \Lambda_{j,j} = \Lambda_{j-1, j-1}$ for all $j \geq 2$.
\begin{center}
    \[\Lambda = \begin{bmatrix}
        1 & 0 & 0 & 0\ldots\\
        (1-\lambda_{1}) & \lambda_{1} & 0 & 0\ldots \\
        (1-\lambda_{1}) & \lambda_{1}(1-\lambda_{2}) & \lambda_{1}\lambda_{2} &0\ldots \\
        (1-\lambda_{1}) & \lambda_{1}(1-\lambda_{2}) & \lambda_{1}\lambda_{2}(1-\lambda_{3}) & \ddots \\
        
        \vdots & \vdots & \vdots& \vdots
    \end{bmatrix}.\]
 \end{center}
Thus,
\begin{equation*}
\begin{split}
    G_{t}^{[\lambda_1:]}(\theta) &- V_{\theta}(s_t)  = R_{t+1} + \gamma V_{\theta}(s_{t+1}) - V_{\theta}(s_{t}) \\
    & + \gamma\lambda_1\{(1-\gamma)(R_{t+2}) + \gamma(1-\gamma)\\&
    [(1-\lambda_{2})(R_{t+2}+V_{\theta}(s_{t+2})) + \lambda_{2}(R_{t+2}+R_{t+3})] \\
    & + \ldots - V_{\theta}(s_{t+1})\} \\
    & = \delta_{t} + \gamma\lambda_{1}[G_{t+1}^{[\lambda_2:]} - V_{\theta}(s_{t+1})] \\ & = \delta_{t} + \gamma\lambda_{1}\delta_{t+1} + \gamma^{2}\lambda_{1}\lambda_{2}\delta_{t+2} + \cdots%\\
    %& = \delta_{t} + \gamma\lambda_{1}\delta_{t+1} + \gamma^{2}\lambda_{1}\lambda_{2}\delta_{t+2} + \ldots \ldots
\end{split}
\end{equation*}
where, $\delta_{t}$ is the TD-error defined as $\delta_t = R_{t+1} + \gamma V_{\theta}(s_{t+1}) - V_{\theta}(s_t)$.
% \begin{equation}
% \delta_t = R_{t+1} + \gamma V_{\theta}(s_{t+1}) - V_{\theta}(s_t).
% \end{equation}
The superscript $[\lambda_{1}:]$ in the return defined above denotes that the two-step flat return is weighed by $\lambda_{1}$, the three-step flat return is weighted by $\lambda_{1}\lambda_{2}$, etc., whereas the superscript $[\lambda_{2}:]$ denotes that the two-step flat return is weighted by $\lambda_{2}$, the three-step flat return by $\lambda_{2}\lambda_{3}$, etc.
Then, for $t = 0, 1, \ldots,$ the {\bf TD($\lambda$)-schedule} algorithm updates $\theta$ as follows: 
\begin{equation}
\label{tdLambdaSchedule}
    \theta_{t+1} = \theta_t + \alpha_t\delta_t z_t, \mbox{ where } z_t = \sum_{k=0}^t\Big(\prod_{j=1}^{t-k}\gamma \lambda_j\Big)\phi(s_k).
\end{equation}
Here $\{\alpha_t\}_{t\geq0}$ is the sequence of step-size parameters and $\theta_0$ is the initial parameter vector.
%\begin{equation}
%     z_t = \sum_{k=0}^t\Big(\prod_{j=1}^{t-k}\gamma \lambda_j\Big)\phi(s_k).
% \end{equation}
We make some key observations here: \textbf{(1)} If $\lambda_j = 1$ for $j \leq n$ and $\lambda_j = 0$ for $j > n$, for some $n > 0$, then we obtain the $n$-step TD algorithm. \textbf{(2)} If $\lambda_{j} = 1$ $\forall j \in \mathbb{N}$, then we obtain the MC algorithm. \textbf{(3)} If $\lambda_{j} = \lambda$ $\forall j \in \mathbb{N}$, then we obtain the TD($\lambda$) algorithm. 
% \begin{itemize}
%     \item If $\lambda_j = 1$ for $j \leq n$ and $\lambda_j = 0$ for $j > n$, for some $n > 0$ , then we obtain the $n$-step TD algorithm.
%     \item If $\lambda_{j} = 1$ $\forall j \in \mathbb{N}$, then we obtain the MC algorithm.
%     \item If $\lambda_{j} = \lambda$ $\forall j \in \mathbb{N}$, then we obtain TD($\lambda$) algorithm.
%     %\item In the definition of $G^{[\lambda_1:]}_t(\theta)$, if we let $\lambda_1$ as $\lambda(s_t)$, $\lambda_2$ as $\lambda(s_{t+1})$ and so on, we get the variable-$\lambda$ algorithm (Section 12.8 \cite{SuttonBarto_book}). Thus, TD($\lambda$)-schedule algorithm is different from the variable-$\lambda$ algorithm, as we specify a schedule of $\lambda$ given by the matrix $\Lambda$, instead of specifying some bootstrapping function $\lambda(\cdot)$ defined from the state-space to $[0,1]$. 
% \end{itemize}

For the remaining part of the paper, we only consider $\lambda$-schedules where $\exists L > 0$ such that $\lambda_j = 0$ for all $j > L$.
We denote the return associated with such a schedule by $G_{t}^{[\lambda_{1}:\lambda_{L}]}(\theta)$.
The equation of $z_t$ then reduces to
\begin{equation}
\label{trace_eq}
 z_t = \sum_{k=t-L}^{t}\left(\prod_{j=1}^{t-k}\gamma\lambda_j\right)\phi(s_k).
\end{equation}
We point out that $z_t$ can't be written recursively in terms of $z_{t-1}$ unlike TD($\lambda$) and therefore using schedules of the form as described above becomes essential to avoid explosion of space. Note that we need to store the last $L$ states to compute the eligibility trace $z_t$. The algorithm for TD($\lambda$)-schedule is given below:
\begin{algorithm}[H]
\begin{algorithmic}[1]
\STATE \textbf{Input:} Policy $\pi$, step-size sequence $\{\alpha_t\}_{t\geq1}$, lambda-schedule $\{\lambda_{t}\}_{t=1}^{L}$ and the feature map $\phi(\cdot)$.
\STATE Initialize $\theta_0$ and $s_0$ randomly, $z \leftarrow 0$.
\FOR{$t$=1,2,3,\ldots}
\STATE Choose an action $a_t \sim \pi(\cdot|s_{t-1})$
\STATE Perform action $a_t$ and observe reward $R_{t+1}$ and next state $s_{t+1}$.
\STATE Compute the eligibility trace as in \eqref{trace_eq}.
\STATE $\delta_t \leftarrow R_t + \gamma \phi(s_{t+1})^{T}\theta - \phi(s_{t})^T\theta$
\STATE $\theta_t \leftarrow \theta_t + \alpha_{t}\delta_t z_t$
\ENDFOR
\end{algorithmic}
\caption{TD($\lambda$)-schedule}
\label{alg:seq}
\end{algorithm}
\subsection{EqualWeights schedule}
As already mentioned, it has been seen empirically that $n$-step TD performs better for some intermediate values of $n$. If for a particular problem, $n$-step TD achieves low MSE for some $n_{1}\leq n \leq n_{2}$ (cf. Figure 7.2 of \cite{SuttonBarto_book}), then it makes sense to combine only these $n$-step returns instead of all the $n$-step returns with exponentially decreasing weights. The TD$(\lambda)$-schedule algorithm lets us do this. Suppose, we want to assign equal weights to all $n$-step returns for $n_{1}\leq n \leq n_{2}$. We can achieve this by selecting a $\lambda$-schedule as follows:
\[\lambda_{i} =
\left\{
	\begin{array}{ll}
		1  & \mbox{if } i < n_1 \\
		1-\frac{1}{n_2-i+1} & \mbox{if } n_1 \leq i \leq n_2\\
		0 & \mbox{if } i > n_2.\\
	\end{array}
\right. 
\]
We call this schedule \emph{EqualWeights}$(n_1,n_2)$. To take an example consider \emph{EqualWeights}$(3,5)$. The $\lambda$-schedule is given by
\(\lambda_{1} = 1,\quad \lambda_{2} = 1, \quad \lambda_{3} = \frac{2}{3}, \quad \lambda_{4} = \frac{1}{2}, \quad \lambda_{j} = 0, \quad\forall j\geq5\) and the associated weight matrix is:
\begin{center}{
    \[\Lambda = 
    \begin{bmatrix}
 1 &         0 &         0 &         0 &         0  & 0 & 0\ldots\\        
 0 &         1 &         0 &         0 &         0  & 0 & 0\ldots\\        
 0 &         0 &         1 &         0 &         0  & 0 & 0\ldots\\        
 0 &         0 &         1/3 &       2/3 &       0 & 0 & 0\ldots \\        
 0 &         0 &         1/3 &       1/3 &       1/3 & 0 & 0\ldots \\
        
        \vdots & \vdots & \vdots& \vdots &\vdots &\ddots &\vdots
    \end{bmatrix}.\]
}\end{center}
We notice that when the episode length $\geq 5$, the above matrix assigns equal weights to $3$-step, $4$-step and $5$-step TD returns.
When the episode length $\leq 3$, it takes the Monte Carlo return. Such an arbitrary weight assignment to $n$-step returns is not possible with TD($\lambda$). Appendix A5 reports evaluation with the EqualWeights schedule on some standard MDPs.

\section{\texorpdfstring{Off Policy Gradient \(\lambda\)}--schedule algorithms}
While TD($\lambda$)-schedule is an on-policy algorithm, we now present a couple of off-policy algorithms that are based on the $\lambda$-schedule procedure.
We first describe the off-policy setting briefly. 
The agent selects actions according to a \emph{behaviour policy} $\mu:\mathcal{S}\times\mathcal{A}\rightarrow[0,1]$, while we are interested in computing the value function $V^{\pi}$ associated with the \emph{target policy} $\pi:\mathcal{S}\times\mathcal{A}\rightarrow[0,1]$.
Let $d^{\mu}(s)$, $s \in \mathcal{S}$ denote the steady-state probabilities for the Markov chain under the behaviour policy $\mu$ and let the importance sampling ratio $\rho_{t} = \frac{\pi(a_{t}|s_{t})}{\mu(a_{t}|s_{t})}$, where $a_{t}$ is the action picked at time-step $t$. 
% Along the lines of per-decision importance sampling (Section 5.9, \cite{SuttonBarto_book}), we can derive the off-policy version of the TD($\lambda$)-Schedule algorithm, using the off-policy return defined as
Along the lines of per-decision importance sampling (Section 5.9, \cite{SuttonBarto_book}), we obtain the off-policy $\lambda$-schedule return as
\begin{equation*}
    \begin{split}
        G_{t}^{[\lambda_1:\lambda_{L}]}(\theta) - V_{\theta}(s_t) &= \rho_{t}\delta_{t} + \gamma\lambda_{1}\rho_{t}\rho_{t+1}\delta_{t+1} \\ & + \cdots+ \gamma^{L}\lambda_{1}\ldots\lambda_{L}\rho_{t}\ldots\rho_{t+L}\delta_{t+L}.
    \end{split}
\end{equation*}
We now obtain the \emph{Off-Policy TD($\lambda$)-schedule} algorithm by defining the eligibility vector $z_{t}$ and the update equation for $\theta$ as below:
{\small
\begin{equation}
    \label{offPolicy_z}
    \theta_{t+1} = \theta_{t} + \alpha_{t}\delta_{t}z_{t} \mbox{, with }
    z_t = \sum_{k=t-L}^t\rho_{t}\Big(\prod_{j=1}^{t-k}\rho_{t-j}\gamma \lambda_j\Big)\phi(s_k).
\end{equation}}
% \begin{equation}
%     \label{offPolicy_theta}
%     \theta_{t+1} = \theta_{t} + \alpha_{t}\delta_{t}z_{t}.
% \end{equation}
We observe that the above algorithm diverges on Baird's Counterexample (\cite{Baird}) (See Appendix A5).
Gradient based algorithms \cite{Maei_PhD} are observed to converge in the off-policy setting. 
Inspired by this, we develop two gradient-based schedule algorithms, \emph{GTD($\lambda$)-Schedule} and \emph{TDC($\lambda$)-Schedule}, as described below. 
%We observe that both these algorithms converge for the Baird's counter-example (See Appendix A5).

% As before, we consider $\lambda$-schedules such that $\lambda_{j} = 0$, $\forall j > L$. 
We note that the $\lambda$-schedule return defined in \eqref{lambda_return} can also be written as below (See Appendix A2):
{\small \begin{equation} \label{return}
G_{t}^{[\lambda_1:\lambda_L]}(\theta) = R_{t+1} + \gamma \left[(1-\lambda_1)V_{\theta}(s_{t+1})+\lambda_1 G_{t+1}^{[\lambda_{2}:\lambda_{L}]}(\theta)\right].
\end{equation}}
Next we define the value function associated with state $s$ as
{\small \begin{equation} \label{eu_eqn}
\begin{split}
V^{\pi}(s) & = \mathbb{E}\left[G_{t}^{[\lambda_1:\lambda_L]}(\theta)|s_{t} = s,\pi\right] \triangleq \left(T^{\pi,[\lambda_1:\lambda_L]}V^{\pi}\right)(s).
\end{split}
\end{equation}}
Here, $T^{\pi,[\lambda_1:\lambda_L]}$ denotes the $\lambda$-schedule Bellman operator. The objective function $J(\theta)$ on which gradient descent is performed is the \emph{Mean Squared Projected Bellman Error} defined as follows:
\begin{equation}
\begin{split}
    J(\theta) & \triangleq ||V_{\theta} - \Pi T^{\pi,[\lambda_1:\lambda_L]}V_{\theta}||_{D_{\mu}}^{2}
    \\ &= \sum_{s} d^{\mu}(s)\left(V_{\theta}(s) - \Pi T^{\pi,[\lambda_1:\lambda_L]}V_{\theta}(s)\right)^{2},
\end{split}
\end{equation}
where $d^{\mu}(s)$ is the visitation probability to state $s$ under the steady-state distribution when the behaviour policy $\mu$ is followed and $D_{\mu}$ is an $n \times n$ diagonal matrix  with $d^\mu(s)$ as its $s^{th}$ diagonal entry. 
We also define
{\small \begin{equation}
    \delta_{t}^{[\lambda_1:\lambda_L]}(\theta) \triangleq G_{t}^{[\lambda_1:\lambda_L]}(\theta) - \theta^T\phi_t,
\end{equation}
\begin{equation}
    P_{\mu}^{\pi}\delta_{t}^{[\lambda_1:\lambda_L]}(\theta)\phi_{t} \triangleq \sum_{s} d^{\mu}(s)\mathbb{E}_{\pi}\left[\delta_{t}^{[\lambda_{1}:\lambda_{L}]}(\theta)\phi_{t}|S_{t} = s, \pi\right].
\end{equation}}
As with \cite{Maei_PhD}, using these definitions, the \emph{Projected Bellman Error} is expressed as the product of three expectations in the following lemma. The proofs of the results below are provided in the Appendix A3.
\begin{lemma}The objective function $J(\theta) = ||V_{\theta} - \Pi T^{\pi,[\lambda_1:\lambda_L]}V_{\theta}||_{D}^{2}$ can be equivalently written as
    $J(\theta) = \left(P_{\mu}^{\pi}\delta_{t}^{[\lambda_1:\lambda_L]}(\theta)\phi_{t}\right)^T \mathbb{E}_{\mu}\left[\phi_{t}\phi_{t}^T\right]^{-1}\left(P_{\mu}^{\pi}\delta_{t}^{[\lambda_1:\lambda_L]}(\theta)\phi_{t}\right).$
\end{lemma}
The above lemma gives an expression for the objective function. However, the expectation is with respect to the target policy $\pi$, but needs to be computed from samples of the trajectory generated by the behaviour policy $\mu$. Secondly, the above is a forward-view  equation and needs to be converted to an equivalent backward view.
Theorem \ref{theorem_1} converts the expectation with respect to $\pi$ to an expectation with respect to $\mu$. In order to do so, as in \cite{Maei_PhD}, we define the following terms:
\begin{equation}
    \begin{split}
        G_{t}^{[\lambda_1:\lambda_L], \rho}(\theta) &= \rho_{t}\Big(R_{t+1} + \gamma\Big( (1-\lambda_{1})V_{\theta}(s_{t+1}) \\ & + \lambda_{1}G_{t+1}^{[\lambda_{2}:\lambda_{L}],\rho}(\theta)\Big)\Big),
    \end{split}
\end{equation}
\[\delta_{t}^{[\lambda_1:\lambda_L],\rho}(\theta) \triangleq G_{t}^{[\lambda_1:\lambda_L], \rho}(\theta) - \theta^T\phi_t.\]
\begin{theorem} \label{theorem_1}
$P_{\mu}^{\pi}\delta_{t}^{[\lambda_1:\lambda_L]}(\theta)\phi_{t} = \mathbb{E}_{\mu} \left[\delta_{t}^{[\lambda_{1}:\lambda_{L}],\rho}(\theta)\phi_{t}\right]$.
\end{theorem}

Theorem \ref{Forward_to_Backward} converts the forward view into an equivalent backward view using the lemma below.
\begin{lemma}
\label{index_change_lemma}
    $\mathbb{E}_{\mu}\left[\rho_{t}\gamma\lambda_{1}\delta_{t+1}^{[\lambda_{2}:\lambda_{L}],\rho}(\theta)\phi_{t}\right] = \mathbb{E}_{\mu}\left[\rho_{t-1}\gamma\lambda_{1}\delta_{t}^{[\lambda_{2}:\lambda_{L}],\rho}(\theta)\phi_{t-1}\right]$.
\end{lemma}

\begin{theorem} 
\label{Forward_to_Backward}
Define the eligibility trace vector $z_{t} = \sum_{i = t-L}^{t}\left[\rho_{t}\left(\Pi_{j=1}^{t-i}\rho_{t-j}\gamma\lambda_j\right)\phi_i\right].$ 
Then,
$\mathbb{E}_{\mu}\left[\delta_{t}^{[\lambda_{1}:\lambda_{L}],\rho}(\theta)\phi_{t}\right] = \mathbb{E}_{\mu}\left[\delta_{t}(\theta)z_{t}\right]$.
\end{theorem}
Using the above results, we can express the objective function as:
\[J(\theta) = \mathbb{E}_{\mu}\left[\delta_t(\theta)z_t\right]^T \mathbb{E}\left[\phi_t\phi_t^T\right]^{-1}\mathbb{E}\left[\delta_t(\theta)z_t\right], \mbox{ hence, }\]
\begin{equation*}
    \begin{split}
        -\frac{1}{2}\nabla J(\theta) &= -\nabla\mathbb{E}\left[\delta_t(\theta)z_t^{T}\right]\mathbb{E}\left[\phi_t\phi_t^T\right]^{-1}\mathbb{E}\left[\delta_t(\theta)z_t\right] \\ &= \mathbb{E}\left[(\phi_t-\gamma\phi_{t+1})z_t^T\right]\mathbb{E}\left[\phi_t\phi_t^T\right]^{-1}\mathbb{E}\left[\delta_t(\theta)z_t\right].
    \end{split}
\end{equation*}
\[\]
% \begin{equation*}
%     J(\theta) = \mathbb{E}_{\mu}\left[\delta_t(\theta)z_t\right]^T \mathbb{E}\left[\phi_t\phi_t^T\right]^{-1}\mathbb{E}\left[\delta_t(\theta)z_t\right]
% \end{equation*}
% \begin{equation*}
% \begin{split}
%     -\frac{1}{2}\nabla J(\theta) = -\nabla\mathbb{E}\left[\delta_t(\theta)z_t^{T}\right]\mathbb{E}\left[\phi_t\phi_t^T\right]^{-1}\mathbb{E}\left[\delta_t(\theta)z_t\right]\\
%     = \mathbb{E}\left[(\phi_t-\gamma\phi_{t+1})z_t^T\right]\mathbb{E}\left[\phi_t\phi_t^T\right]^{-1}\mathbb{E}\left[\delta_t(\theta)z_t\right]\\
% \end{split}
% \end{equation*}
We keep a stationary average for the second and third expectations in a parameter vector $w$ and sample the terms in the first expectation. We call the resultant algorithm {\bf GTD($\lambda$)-schedule} whose iterates are as given below:
\begin{equation}
\label{GTDLambdaSchedule}
\begin{split}
    \theta_{t+1} = \theta_t + \alpha_t\left((\phi_t-\gamma\phi_{t+1})z_t^T w_t\right),\\
    w_{t+1} = w_t + \beta_t\left(\delta_{t}(\theta_t)z_{t}-\phi_t\phi_t^{T}w_{t} \right).
    \end{split}
\end{equation}
Next, as with \cite{Maei_PhD}, an alternative is to express the gradient direction as:
\begin{equation*}
\begin{split}
-\frac{1}{2}\nabla J(\theta) %\mathbb{E}\left[(\phi_t-\gamma\phi_{t+1})z_t^T\right]\mathbb{E}\left[\phi_t\phi_t^T\right]^{-1}\mathbb{E}\left[\delta_t(\theta)z_t\right]\\
 & = \left(\mathbb{E}\left[\phi_t\phi_{t}^{T}\right] + \mathbb{E}\left[(\phi_t-\gamma\phi_{t+1})z_t^T - \phi_t\phi_{t}^{T}\right]\right)\\& \qquad\mathbb{E}\left[\phi_t\phi_t^T\right]^{-1}\mathbb{E}\left[\delta_t(\theta)z_t\right]\\ 
 & = \mathbb{E}\left[\delta_t(\theta)z_t\right] - \left(\mathbb{E}\left[(\gamma\phi_{t+1}-\phi_t)z_{t}^{T} + \phi_t\phi_{t}^{T}\right]\right)\\ & \qquad\left(\mathbb{E}\left[\phi_t\phi_{t}^{T}\right]^{-1}\mathbb{E}\left[\delta_t(\theta)z_t\right]\right).
\end{split}
\end{equation*}
As before, we maintain a stationary estimate for the last two terms and sample the remaining terms to obtain the iterates for {\bf TDC($\lambda$)-schedule}:
\begin{equation}
\label{TDCLambdaScehdule}
    \begin{split}
        \theta_{t+1} & = \theta_t + \alpha_t\left(\delta_t(\theta_{t})z_t\right) \\& - \alpha_t\left((\gamma\phi_{t+1}-\phi_{t})z_{t}^{T}w_{t} + \phi_{t}\phi_{t}^{T}w_{t}\right),\\
        w_{t+1} & = w_t + \beta_t\left(\delta_{t}(\theta_{t})z_{t}-\phi_t\phi_t^{T}w_{t} \right).
    \end{split}
\end{equation}
Appendix A5 compares GTD($\lambda$)-schedule and TDC($\lambda$)-schedule with GTD and TDC.

\section{Convergence Analysis}
\label{convergence}
%We borrow ideas in the initial portion of our proof from \cite{TDLambda}. 
Our proof technique differs significantly from other references in the asymptotic analysis of our algorithm. 
In particular, we follow the ordinary differential equation (ODE) based analysis under Markov noise for single and multiple timescale algorithms (cf.~\cite{BorkarBook}, \cite{Arunselvan1}, \cite{chandru-SB} and \cite{prasenjit-SB}). We begin with the convergence analysis
of the TD($\lambda$)-Schedule algorithm. 
Starting from some initial state $s_0$, we generate a single infinitely long trajectory $(s_0, s_1, \ldots)$.
Suppose at time $t$, value of the parameter $\theta$ is $\theta_t$.
We consider a linear parameterisation of the value function as 
$V_{\theta}(s_t) = \phi_t^T\theta_t$, where $\phi_{t} \equiv \phi(s_t)$.
After the transition from state $s_t$ to $s_{t+1}$, we evaluate the temporal difference term and update the parameter $\theta_t$ according to
\eqref{tdLambdaSchedule}, assuming the product $\prod_{j=n+1}^{n} = 1$, $\forall n$.

%\begin{equation}
%\label{tdLambdaSchedule}
%\begin{split}
%\delta_t = R_{t+1} + \gamma V_{\theta}(s_{t+1}) - V_{\theta}(s_t), \\
%\theta_{t+1} = \theta_t + \alpha_t\delta_t z_t, \\
%z_t = \sum_{k=0}^{t}\left(\prod_{j=1}^{t-k}\gamma\lambda_j\right)\phi(s_k),
%\end{split}
%\end{equation}
%where we define $\prod_{j=n+1}^{n} = 1$ $\forall n$.
As mentioned above, we only consider $\lambda$ schedules where $\exists L > 0$ such that $\lambda_j = 0$ for all $j > L$.
With such a choice of schedule, we need to store only the last $L$ states. 
% The equation of $z_t$ then reduces to
% \begin{equation}
%  z_t = \sum_{k=t-L}^{t}\left(\prod_{j=1}^{t-k}\gamma\lambda_j\right)\phi(s_k).
% \end{equation}
% We point out that $z_t$ can't be written recursively in terms of $z_{t-1}$ unlike TD($\lambda$) and therefore using schedules of the form as described above becomes essential to avoid explosion of space.
We make the following assumptions:
\begin{assumption} 
\label{assumption1}
The step-sizes $\alpha_t$ are positive and satisfy $\sum_t \alpha_t = \infty$ and $\sum_t \alpha_t^2 < \infty$.
\end{assumption}

\begin{assumption}
\label{assumption2}
There exists a distribution $d^{\pi}(j),j\in S$ such that 
\(
    \lim_{t \rightarrow \infty} P(s_t = j|s_0 = i) = d^{\pi}(j) \mbox{         }\forall i,j.
\)
\end{assumption}

\begin{assumption}
\label{assumption3}
The matrix $\Phi$ has full rank, where $\Phi$ is an $n\times d$ matrix where the s$^{th}$ row is $\phi(s)^T$.
\end{assumption}

Let $X_t$ = $(s_{t-L},s_{t-L+1},\ldots,s_{t},s_{t+1})$. Clearly $X_t$ is a Markov chain because $s_{t+2}$ only depends on $s_{t+1}$.
Note that $z_{t}$ is not included in the Markov chain as it can be constructed from $X_{t}$. 
The steady state version of the Markov chain can be constructed from $s_t$, -$\infty < t < \infty$, whose transition probabilities are given by $P$. 
We then let 
\begin{equation*}
    \begin{split}
        z_{t}  = \sum_{\tau = -\infty}^{t}\left(\prod_{j=1}^{t-\tau}\gamma\lambda_{j}\phi(s_{\tau})\right)
         = \sum_{\tau = t-L}^{t}\left(\prod_{j=1}^{t-\tau}\gamma\lambda_j\phi(s_{\tau})\right).
    \end{split}
\end{equation*}
We use $\mathbb{E}_{0}[\cdot]$ to denote the expectation with respect to the steady-state distribution of $X_t$. 
Now, we can write $\delta_tz_t$ as: $\delta_t z_t = A(X_t)\theta_t + b(X_t)$,
where, 
$b(X_t) = z_t R_{t+1}$ and $A(X_t) = z_t(\gamma\phi_{t+1}^{T} - \phi_{t}^T)$. Let $D$ be the diagonal matrix with $d^\pi(s)$, $s\in S$ as it's diagonal elements. Further, let
$A = \mathbb{E}_{0}[A(X_t)]$ and  $b = \mathbb{E}_{0}[b(X_t)]$.

%Let $||\cdot||_D$ be the weighted quadratic norm defined by $||V||_D = \sqrt{V^TDV} = \sqrt{\sum_{s=1}^{n}d^{\pi}(s)V(s)^{2}}$.  
% \begin{lemma}
% \label{lemma7_tsitsiklis}
% $\mathbb{E}_{0}[\phi(s_0)\phi(s_m)^T] = \Phi^TDP^m\Phi$
% \end{lemma}
% \begin{proof}
% Refer to Lemma 7 in \cite{TsitsiklisVanRoy}.
% \end{proof}
% First, consider two vectors $V$ and $\bar{V}$. Then we have, 
% \begin{equation}
% \begin{split}
%  \mathbb{E}_{0}[V(s_0)\bar{V}(s_m)]
%   & = \sum_{i=1}^n d^{\pi}(i) \sum_{j=1}^n P(s_m=j|s_0=i)V(i)\bar{V}(j) \\
%  &= \sum_{i=1}^n d^{\pi}(i) V(i) [P^m\bar{V}](i) \\
%  &= V^TDP^m\bar{V}
% \end{split}
% \end{equation}

% where $[P^m\bar{V}](i)$ denotes the $i^{th}$ component of the vector $P^m\bar{V}$.
% For the case where $V=\Phi\theta$ and $\bar{V}=\Phi\bar{\theta}$, we get 
% \begin{equation}
%      \mathbb{E}_{0}[\theta^T \phi(s_0) \phi(s_m)^T \bar{\theta}] = \theta^T \Phi^TDP^m\Phi\bar{\theta}.
% \end{equation}

% And since, $\theta$ and $\bar{\theta}$ are arbitrary, we get 
% \begin{equation}
%      \mathbb{E}_{0}[\phi(s_0)\phi(s_m)^T] = \Phi^TDP^m\Phi
% \end{equation}

\begin{proposition}
\label{propND}
    The matrix $A$ is negative definite.
\end{proposition}
\begin{proof}
See Appendix A4.
\end{proof}
\subsection{\texorpdfstring{Convergence of TD(\(\lambda\))}--schedule}
We now present a result from Chapter 6 of \cite{BorkarBook} (see also \cite{Arunselvan1}) that gives the stability and convergence of a stochastic approximation recursion under Markov noise. 
%Recall that $S$ denotes the state space (the set in which the state-valued process $\{s_t\}$ takes values). 
Let $\check{S}$ denote the set in which the process $\{X_t\}$ takes values in.
\begin{theorem}
\label{theorem1} Consider the following recursion in $\mathbb{R}^d$:
\begin{equation}
    \label{arunSA}
    \theta_{t+1} = \theta_{t} + \alpha_{t}\left( h(\theta_{t},X_{t}) + M_{t+1}\right).
\end{equation}
Consider now a sequence $\{t(n)\}$ of time points defined as follows:
$t(0)=0$, ${t(n) = \sum_{k=0}^{n-1} \alpha_k}$, $n\geq 1$.
Now define the algorithm's trajectory $\bar{\theta}(t)$ according to: $\bar{\theta}(t(n)) = \theta_n$, $\forall n$, and
with $\bar{\theta}(t)$ defined as a continuous linear interpolation on all intervals $[t(n),t(n+1)]$.
Finally, consider the following assumptions:
% under the following assumptions:
\begin{itemize}
    \item[{\bf (B1)}]  $h:\mathbb{R}^d \times \check{S} \rightarrow \mathbb{R}^d$ is Lipschitz continuous in the first argument, uniformly with respect
    to the second.
    % \item[{\bf (B2)}] $\{X_t\}$ is a $\check{S}$-valued Markov process where $\check{S} \subseteq \mathbb{R}^m$ satisfying
    % \[ P(X_{t+1} \in A \mid X_s, \theta_s, s\leq t) = \int_{A} p(dx\mid X_s, \theta_s), \mbox{ } t\geq 0,\]
    % for $A$ Borel in $\check{S}$ with $(x,\theta)\in \check{S}\times \mathbb{R}^d \mapsto p(dw\mid x,\theta)$ a continuous map specifying the transition probability kernel.
    % \item[{\bf (B3)}] 
    % \begin{itemize}
    %      \item[(a)] If $\check{S} \subset \mathbb{R}^m$, then $\check{S}$ is a compact subset of $\mathbb{R}^m$.
    % \item[(b)] If $\check{S}=\mathbb{R}^m$ above, there exists a stochastic Lyapunov function $V\in C(\mathbb{R}^m)$ such that\\
    %     (i) $\lim_{\parallel x\parallel \rightarrow\infty} V(x) =\infty$,\\
    %     (ii) $\sup_n E[V(X_n)^2] <\infty$,\\
    %   (iii) For a compact set $B\subset \check{S}$ and $\epsilon>0$, $E[V(X_{n+1}) \mid \mathcal{F}_n] 
    %     \leq V(X_n) -\epsilon$ almost surely on $\{X_n\not\in B\}$.
    %     \end{itemize}

    \item[{\bf (B2)}] For any given $\theta\in \mathbb{R}^d$, the set $D(\theta)$ of ergodic occupation measures of $\{X_n\}$ is compact and convex.  
    \item[{\bf (B3)}] $\{M_{t}\}_{t \geq 1}$ is a square-integrable martingale difference sequence. 
Further, $\mathbb{E}\left[|| M_{t+1}||^{2}|\mathcal{F}_{t}\right] \leq K(1+||\theta_{t}||^{2})$, where 
$\mathcal{F}_t = \sigma(\theta_{m},X_{m},M_{m},m \leq t)$, $t\geq 0$.
\item[{\bf (B4)}] The step-size sequence $\{\alpha_{t}\}_{t \geq 0}$ satisfies $\alpha_t>0, \forall t$. Further,
$\sum_{t=0}^{\infty}\alpha_{t} = \infty$ and $\sum_{t=0}^{\infty}\alpha_{t}^{2} < \infty$.
\item[{\bf (B5)}] Let ${\displaystyle \tilde{h}(\theta,\nu) = \int h (\theta,x)\nu(dx)}$. Also,
${\displaystyle h_c(\theta,\nu) = \frac{\tilde{h}(c\theta, \nu(c\theta))}{c}}$. 
\begin{itemize}
\item[(i)] The limit ${\displaystyle 
\tilde{h}_\infty(\theta,\nu) = \lim_{c\rightarrow\infty} \tilde{h}_c(\theta,\nu)}$ exists uniformly on compacts.
 \item[(ii)] There exists an attracting set $\mathcal{A}$ associated with the differential inclusion (DI)
 $\dot{\theta}(t) \in H(\theta(t))$ where $H(\theta) = \bar{co}(\{\tilde{h}_\infty(\theta,\nu): \nu\in D(\theta)\})$
 such that $\sup_{u\in \mathcal{A}} ||u|| < 1$ and $\bar{B}_1(0) \stackrel{\triangle}{=} \{x\mid ||x|| \leq 1\}$
 is a fundamental neighborhood of $\mathcal{A}$.
\end{itemize}
\end{itemize}
Under (B1)-(B5), $\{\bar{\theta}(s+\cdot), s\geq 0\}$ converges to an internally chain transitive invariant set of the differential inclusion
\(\dot{\theta}(t) \in \hat{h}(\theta(t)),\)
where $\hat{h}(\theta) = \{\tilde{h}(\theta,\nu)\mid \nu \in D(\theta)\}$. In particular,
$\{\theta_t\}$ converges almost surely to such a set.
\end{theorem}
We now present our main result on the TD($\lambda$)-schedule algorithm.

\begin{theorem}
\label{TDLS}
Under Assumptions 1--3, the TD($\lambda$)-schedule algorithm given by \eqref{tdLambdaSchedule} satisfies $\theta_t \rightarrow \theta^*\stackrel{\triangle}{=}
-A^{-1}b$ almost surely as $t\rightarrow\infty$.
\end{theorem}

\begin{proof}
We first transform the iterate for the TD($\lambda$)-schedule algorithm given by \eqref{tdLambdaSchedule} in the standard stochastic approximation form given by \eqref{arunSA}.
Note that the algorithm \eqref{tdLambdaSchedule} can be rewritten as
\begin{equation}
    \label{tdLnew}
    \theta_{t+1} = \theta_t + \alpha_t(A(X_t)\theta_t + b(X_t)).
\end{equation}
Thus, $h(\theta_t,X_t) = A(X_t)\theta_t + b(X_t)$.
%Since the Markov process $\{X_t\}$ takes values in a finite state space,
%for (B1) to hold, we only need to verify that the function $h$ is Lipschitz continuous
%in $\theta$. 
For any $\theta_1,\theta_2\in \mathbb{R}^d$ and $X\in \check{S}$,
\begin{equation*}
    \begin{split}
        \parallel h(\theta_1,X)-h(\theta_2,X)\parallel &\leq \parallel A(X)(\theta_1-\theta_2)\parallel \\ &\leq \parallel A(X)\parallel \parallel \theta_1-\theta_2\parallel.
    \end{split}
\end{equation*}
% It is now easy to see that $\delta_t$ and $z_t$ are both uniformly bounded.
Since the set $S$ is finite, $\exists 0<M<\infty$ such that $\sup_t ||\phi_t|| \leq M$, thus $z_t$ is bounded.
Hence, $\exists 0<B<\infty$ such that $\sup_{X\in S} ||A(X)|| \leq B$
and (B1) holds.
%We now consider (B2).
%Note that the parameter $\theta$ is used for value function estimation and does
%not in any way influence the transition probabilities. Moreover, the process $\{X_t\}$ is a discrete state Markov chain in our case with
%\[P(X_{t+1}=x\mid X_s,s\leq t) = p(X_{t+1}=x\mid X_t) \mbox{ a.s.}\]
%The continuous kernel requirement is not there in our case.
%The requirement (B3) is for more more general state spaces such as $\mathbb{R}^m$ or
%a compact subset of the same. In particular, (B3)(b) ensures the stability of the
%Markov process in $\mathbb{R}^m$. In our setting since $\check{S}$ is a finite state
%space, (B3) is not required in our case.
From Assumption 2, the Markov chain $\{s_t\}$ has a unique stationary distribution. 
Hence, the Markov chain $\{X_t\}$ also has a unique stationary distribution $\nu$. 
Since the transition probabilities of $\{X_t\}$ do not depend on $\theta$, we have
$D(\theta)\equiv D =\{\nu\}$, a singleton, for all $\theta\in \mathbb{R}^d$. The set $D$ is trivially compact and convex, establishing (B2).
Now upon comparison of \eqref{tdLnew} with \eqref{arunSA}, we observe that
$M_{t+1}= 0$, $\forall t$ in \eqref{tdLnew}. Thus, (B3) is trivially verified as well
with $\mathcal{F}_{t} = \sigma(\theta_{0},\{s_{m},R_{m}\}_{m\leq t})$, $t\geq 0$.
Further, the requirement (B4) on step-sizes is satisfied from Assumption 1.
We finally consider the requirement (B5). It is easy to see
that
$\tilde{h}(\theta,\nu) = A\theta + b$.
%where, as described earlier, $A=\mathbb{E}_0[A(X_t)]$ and $b=\mathbb{E}_0[b(X_t)]$, respectively. 
Further,
${\displaystyle \tilde{h}_c(\theta,\nu) = \frac{\tilde{h}(c\theta,\nu(c\theta))}{c} = A\theta + \frac{b}{c}}$.
Finally, it is easy to see that
$\tilde{h}_c(\theta,\nu) \rightarrow \tilde{h}_\infty(\theta,\nu) = A\theta$, as
$c\rightarrow \infty$.
The differential inclusion $\dot{\theta}(t) \in H(\theta)$ now corresponds to
the ODE
\begin{equation}
    \label{ode-is}
\dot{\theta}(t) = \tilde{h}_\infty(\theta,\nu) = A\theta, 
\end{equation}
since as described above, the set $D$ of stationary probability measures is a singleton that does not change with $\theta$. From Proposition~\ref{propND}, $A$ is negative definite.
Hence, the ODE \eqref{ode-is} has the origin as it's unique globally asymptotically stable equilibrium. Thus, $\mathcal{A} =\{0\}$ serves as the attractor for \eqref{ode-is} with $\bar{B}_1(0)$ or for that matter any closed subset of $\mathbb{R}^d$ containing the origin as it's fundamental neighborhood. 
Thus, by Theorem~\ref{theorem1}, we have that $\{\theta_t\}$ converges almost surely to an internally chain transitive invariant set of the ODE 
\begin{equation}
    \label{ode-f1}
\dot{\theta}(t) = \tilde{h}(\theta,\nu)
= A\theta+b.
\end{equation}
Consider now the function ${\displaystyle V(\theta) = \frac{1}{2} (A\theta+b)^T
(A\theta+b)}$. For this function,
\(\frac{dV(\theta)}{dt} = \nabla V(\theta)^T \dot{\theta} = (A\theta+b)^T A^T (A\theta+b).\)
Since $A$ is negative definite (cf.~Proposition~\ref{propND}), it follows that ${\displaystyle \frac{dV(\theta)}{dt} <0}$ if $\theta\not=\theta^*$
and ${\displaystyle \frac{dV(\theta)}{dt} =0}$ otherwise.
%\begin{eqnarray*}
%\frac{dV(\theta)}{dt} 
 %    & < 0 \mbox{ if } \theta\not=\theta^*\\
  %   & =0 \mbox{ if } \theta=\theta^*.
%\end{eqnarray*}
Thus, $V:\mathbb{R}^d\rightarrow \mathbb{R}$ serves as a Lyapunov function
for the ODE \eqref{ode-f1} with $\theta^*=-A^{-1}b$ as it's unique globally asymptotically
stable attractor. The singleton set $\{\theta^*\}$ is trivially connected and internally chain recurrent and so by Proposition 5.3 of \cite{Benaim99} is also internally chain transitive. 
It is also the only invariant set for the ode (\ref{ode-f1}). %\eqref{ode-f}. 
Thus, by Theorem ~\ref{theorem1}, $\theta_t\rightarrow\theta^*$ w.p.~1 as $t\rightarrow\infty$.
\end{proof}

\subsection{ \texorpdfstring{Convergence of GTD(\(\lambda\))}--schedule and \texorpdfstring{TDC(\(\lambda\))}--schedule}
We make the following assumption for the convergence analysis of GTD($\lambda$)-schedule.
\begin{assumption}
The step-size sequence $\beta_t$ satisfies $\beta_{t} > 0,$ $\forall t,$ $\sum_{t}\beta_t = \infty$ and $\sum_{t}\beta_{t}^2<\infty$. Further, we assume \(\frac{\beta_k}{\alpha_k} = \eta\;\; \forall k \geq 0.\)
\end{assumption}
\begin{theorem}
\label{gtd_thm}
Under Assumptions 1-4, $\{\theta_{t}\}$ in the GTD($\lambda$)-Schedule iterate given in equation \eqref{GTDLambdaSchedule} converges almost surely to $-A^{-1}b$.
\end{theorem}
\begin{proof}
See Appendix A4 for proof of Theorem \ref{gtd_thm}.
\end{proof}
% %%%%%%%%%%%%%%%%%
% %%%%%%%%%%%%%%%%%%%
% \subsection{Convergence of TDC($\lambda$)-Schedule}
We now make the following assumption for the convergence analysis of TDC($\lambda$)-schedule.
\begin{assumption}
The step-size sequence $\beta_t$ satisfies $\beta_{t} > 0,$ $\forall t,$ $\sum_{t}\beta_t = \infty$ and $\sum_{t}\beta_{t}^2<\infty$. Further, we assume, \(\frac{\alpha_k}{\beta_k}\rightarrow0 \mbox{ as } k\rightarrow\infty.\)
\end{assumption}
The TDC($\lambda$)-schedule update rule can be rewritten in the form:
\begin{equation}
\label{TDCLambdaSchedule1}
    \begin{split}
        \theta_{t+1} & = \theta_t + \alpha_t h(\theta_t,w_t,X_t),\\
        w_{t+1} & = w_t + \beta_t g(\theta_t,w_t,X_t),
    \end{split}
\end{equation}
where $h:\mathbb{R}^d\times\mathbb{R}^d\times \check{S} \rightarrow \mathbb{R}^d$ and $g:\mathbb{R}^d\times\mathbb{R}^d\times \check{S} \rightarrow \mathbb{R}^d$ are defined as
$h(\theta_t,w_t,X_t) = A(X_t)\theta_t + b(X_t) - A(X_t)^{T} w_t - C(X_t) w_t$ and
$g(\theta_t,w_t,X_t) = A(X_t)\theta_t + b(X_t) - C(X_t)w_t$, respectively.
Our analysis here is based on stability and convergence results of two-timescale stochastic approximation from \citet{chandru-SB} and \citet{prasenjit-SB}. 

%\begin{equation}
%\label{TDCLambdaScehdule}
 %   \begin{split}
  %      \theta_{t+1} & = \theta_t + \alpha_t\left(\delta_t(\theta_{t})z_t\right) - \alpha_t\left((\gamma\phi_{t}'-\phi_{t})z_{t}^{T}w_{t} + \phi_{t}\phi_{t}^{T}w_{t}\right)\\
  %     w_{t+1} & = w_t + %\beta_t\left(\delta_{t}(\theta_{t})z_{t}-(w_{t}^{T}\phi_t)\phi_t\right)
%    \end{split}
%\end{equation}
%\vspace*{10pt}
Define functions $\bar{h}, \bar{g}: \mathbb{R}^d\times\mathbb{R}^d\rightarrow \mathbb{R}^d$ according to
\[\bar{h}(\theta,w) = \int h(\theta,w,X)\nu(dX) = A\theta + b-A^Tw-Cw, \]
\[\bar{g}(\theta,w) = \int g(\theta,w,X)\nu(dX) = A\theta+b-Cw, \]
respectively, with $A, b, C$ as before.
%Recall here that $\nu$ is the unique stationary invariant distribution of the process $\{X_t\}$.
We shall first present below the main result for which we need the following assumptions:

\begin{itemize}
    \item[{\bf (C1)}] The functions $h(\theta,w,X)$ and $g(\theta,w,X)$ are Lipschitz continuous in $(\theta,w)$ for given $X\in \check{S}$.
    \item[{\bf (C2)}] $\{\alpha_t\}$ and $\{\beta_t\}$ are step-size schedules that satisfy: $\alpha_t,\beta_t>0$, $\forall t$,
    \(\sum_t\alpha_t =\sum_t\beta_t=\infty, \)\( \sum_t(\alpha^2_t+\beta^2_t)<\infty\), \({\displaystyle\lim_{t\rightarrow\infty} \frac{\alpha_t}{\beta_t}=0}.\)
    \item[{\bf (C3)}] The sequence of functions
    ${\displaystyle \bar{g}_c(\theta,w) \stackrel{\triangle}{=} \frac{\bar{g}(c\theta, cw)}{c}}$, $c\geq 1$, satisfy $\bar{g}_c\rightarrow \bar{g}_\infty$ uniformly on compacts for some $\bar{g}_\infty:\mathbb{R}^d\rightarrow\mathbb{R}^d$. Also, the limiting ODE $\dot{w}(t) =\bar{g}_\infty(\theta, w(t))$, i.e., with $\theta(t)\equiv\theta$, has a unique globally asymptotically stable equilibrium $\lambda_\infty(\theta)$ where $\lambda_\infty:\mathbb{R}^d\rightarrow\mathbb{R}^{d}$ is Lipschitz continuous. Further, $\lambda_\infty(0)=0$, i.e., $\dot{w}(t) = \bar{g}_\infty(0,w(t))$ has the origin in $\mathbb{R}^{d}$ as it's unique globally asymptotically stable equilibrium. 
    \item[{\bf (C4)}] The functions $\bar{h}_c(\theta) \stackrel{\triangle}{=} \frac{\bar{h}(c\theta, c\lambda_\infty(\theta))}{c}$, $c\geq 1$ satisfy $\bar{h}_c\rightarrow \bar{h}_\infty$ as $c\rightarrow\infty$ uniformly on compacts for some $\bar{h}_\infty:\mathbb{R}^d\rightarrow \mathbb{R}^d$. Also, the limiting ODE $\dot{\theta}(t) = \bar{h}_\infty(\theta(t))$ has the origin in $\mathbb{R}^d$ as it's unique globally asymptotically stable equilibrium. 
    \item[{\bf (C5)}] The ODE $\dot{w}(t) = \bar{g}(\theta,w(t))$ has a globally asymptotically stable equilibrium $\lambda(\theta)$ (uniformly in $\theta$), where $\lambda: \mathbb{R}^d\rightarrow \mathbb{R}^{d}$ is Lispchitz continuous. \item[{\bf (C6)}] The ODE
    $\dot{\theta}(t) = \bar{h}(\theta(t),\lambda(\theta(t)))$ has a globally asymptotically stable equilibrium $\theta^*$.
\end{itemize}

% The settings considered in \citet{chandru-SB} and \citet{prasenjit-SB} are for general stochastic approximation recursions.  
% Theorem 10 of \citet{chandru-SB} shows that both the iterates remain stable almost surely under similar assumptions as (C1)-(C4). The analysis there is however carried out for the case when the noise sequence is a martingale difference. The same will also work for the case of Markov noise following similar arguments as in \citet{Arunselvan1} and \citet{BorkarBook}. Further, in Theorem 2.6 of \citet{prasenjit-SB}, the convergence of two-timescale stochastic approximation under Markov noise is shown assuming that both the iterates remain stable, i.e., under Assumptions (C1), (C2), (C5), (C6) and under the additional requirement of iterate stability. Assumptions (C3)-(C4) in addition to (C1)-(C2) are sufficient requirements to show the stability of the iterates (see \citet{chandru-SB}). Thus the above conditions are sufficient to show both the stability and convergence of two-timescale stochastic approximation iterates.
We now state the key result.
\begin{theorem}
\label{MainResult}
Under (C1)-(C6), the recursions \eqref{TDCLambdaSchedule1} satisfy:
(a) $\sup_{n} (||\theta_n|| + ||w_n||)<\infty$ and (b) $(\theta_n,w_n) \rightarrow (\theta^*, \lambda(\theta^*))$ almost surely.
\end{theorem}

We now have our main result for the TDC($\lambda$)-schedule algorithm.
\begin{theorem}\label{thm4.5}
Under Assumptions 1-3 and 5, the TDC($\lambda$)-schedule algorithm given by \eqref{TDCLambdaScehdule} satisfies $\theta_{t}\rightarrow\theta^*=-A^{-1}b$ almost surely as $t\rightarrow\infty$.
\end{theorem}
\begin{proof}

From Theorem~\ref{MainResult}, we need to verify that conditions (C1)-(C6) hold.
Note that by definition,
\[ ||h(\theta_2,w_2,X)-h(\theta_1,w_1,X)|| \leq
||A(X)||||\theta_2-\theta_1||\]
\[ + ||A(X)^T|| ||w_2-w_1|| + ||C(X)|| ||w_2-w_1||.
\]
As with the proof of Theorem 4.3, we have that $\sup_{X\in \check{S}} ||A(X)|| \leq B<\infty$. Similarly, since $\phi_t =\phi(s_t)$ with $s_t\in S$, a finite set, $\parallel C(X_t)\parallel$ is also uniformly upper bounded. 
Without loss of generality, let $\sup_{X\in \check{S}} \parallel C(X)\parallel \leq B$ as well. Also, by Proposition A.25 of \citet{BertTsit}, $\parallel A(X)\parallel = \parallel A(X)^T\parallel \leq B$ for Euclidean norms. It is easy to see now that $h(\theta,w,X)$ is Lipschitz continuous in $(\theta,w)$ for given $X\in\check{S}$. The proof that $g(\theta,w,X)$ is Lispchitz continuous in $(\theta,w)$ for given $X\in\check{S}$ is analogous. Thus (C1) holds.

Condition (C2) is just a combination of Assumptions 1 and 5 and thus holds.

For (C3), note that $\bar{g}(\theta,w) = A\theta+b -Cw$. 
Also, $\bar{g}_c(\theta,w) = A\theta-Cw + \frac{b}{c}$. 
We have that $\bar{g}_c(\theta,w) \rightarrow \bar{g}_\infty(\theta,w)=A\theta-Cw$ as $c\rightarrow\infty$. 
Consider now the ODE $\dot{w}(t) = \bar{g}_\infty(\theta,w(t)) = A\theta-Cw(t)$. 
Now notice that $C=\mathbb{E}_0[\phi_t\phi_t^T] = \sum_s d^\pi(s) \phi(s)\phi(s)^T = \Phi^T D\Phi$. 
Since $D$ is a diagonal matrix with positive diagonal elements, it is positive definite. 
Moreover, by Assumption 3, $\Phi$ has full rank. Thus, $C$ is positive definite as well. 
It is now easy to see that
$V_1(w) = (A\theta-Cw)^T(A\theta-Cw)/2$, $w\in \mathbb{R}^d$ serves as a Lyapunov function for $\dot{w}(t) = A\theta-Cw(t)$ with
$\lambda_\infty(\theta) = C^{-1}A\theta$ as the globally asymptotically stable attractor.
%\[\frac{dV_1(w)}{dt} = \nabla V_1(w)^T \dot{w} = (Cw-A\theta)^TC(A\theta-Cw)     \]
%\begin{eqnarray*}
%& < 0 \mbox{ if } w\not= C^{-1}A\theta\\
%& = 0 \mbox{ if } w = C^{-1} A\theta.
%\end{eqnarray*}
%Thus, $V_1$ serves as a Lyapunov function for the aforementioned ODE. Here, $\lambda_\infty(\theta)=C^{-1}A\theta$. 
Now $\lambda_\infty(\theta)$ is a linear function of $\theta$, hence Lipschitz continuous. Further, $\lambda_\infty(0)=0$, whereby the ODE $\dot{w}=\bar{g}_\infty(0,w(t))$ has the origin in $\mathbb{R}^d$ as it's unique globally asymptotically stable equilibrium. Thus, (C3) is verified.

For verifying (C4), consider the functions 
\[ \bar{h}_c(\theta) = \frac{\bar{h}(c\theta,c\lambda_\infty(\theta))}{c} = 
A\theta -A^T C^{-1}A\theta - CC^{-1}A\theta + \frac{b}{c}\]
\[= - A^T C^{-1} A\theta + \frac{b}{c} \longrightarrow - A^T C^{-1}A\theta \stackrel{\triangle}{=}
\bar{h}_\infty(\theta) \mbox{ as } c\rightarrow \infty.
\]
Consider now the ODE
\begin{equation}
    \label{ode-f}
\dot{\theta}(t) = \bar{h}_\infty(\theta(t)) = -A^T C^{-1} A\theta.
\end{equation}
Note that $A$ is negative definite (cf.~Proposition~4.1), hence full-rank and $C$ has been shown above to be positive definite, for any $x\not=0$, $x\in \mathbb{R}^d$, $x^T(-A^T C^{-1} A) x <0$. Hence, $-A^T C^{-1}A$ is a negative definite matrix and \eqref{ode-f} has the origin in $\mathbb{R}^d$ as it's unique globally asymptotically stable attractor. Thus, (C4) is also verified.
Now consider the ODE: \(\dot{w}(t) = \bar{g}(\theta,w(t)) = A\theta + b - C w(t).\)
% Now consider the ODE:
% \begin{equation}
%     \label{ode-g}
% \dot{w}(t) = \bar{g}(\theta,w(t)) = A\theta + b - C w(t).
% \end{equation}
Here
$V_2(w) = \frac{1}{2} (A\theta+b-Cw)^T (A\theta+b-Cw)$ can be easily seen to be a Lyapunov function for the above ODE with 
$\lambda(\theta) = C^{-1}(A\theta+b)$ as the unique globally asymptotically stable equilibrium.
%Then, 
%\[\frac{dV_2(w)}{dt} = \nabla V_2(w)^T \dot{w} 
%= -(A\theta+b-Cw)^T C^T(A\theta+b-Cw)
%\]
%\begin{eqnarray*}
%& <0 \mbox{ if } w \not=\lambda(\theta)\\
%& =0 \mbox{ if } w=\lambda(\theta)
%\end{eqnarray*}
%This follows because $-C$ is a negative definite matrix. In the above,
%$\lambda(\theta) = C^{-1}(A\theta+b)$. 
$\lambda(\theta)$ is also a linear function of $\theta$, hence is Lipschitz continuous. This verifies (C5).

Finally, consider the ODE
\[ 
  \dot{\theta}(t) = \bar{h}(\theta(t),\lambda(\theta(t))) 
    = A\theta(t) + b -A^T C^{-1}(A\theta(t)+b) \]
\begin{equation}
    \label{ode-h}
- C C^{-1}(A\theta(t)+b)
 = -A^T C^{-1} (A\theta(t) +b).
\end{equation}
The function
$V_3(\theta) \stackrel{\triangle}{=} (A^T C^{-1}(A\theta+b))^T
(A^T C^{-1}(A\theta+b))/2$ serves as a Lyapunov function for \eqref{ode-h} with $\theta^* = -A^{-1}b$ as it's
unique globally asymptotically stable equilibrium.
%Then,
%\begin{equation*}
 %   \begin{split}
  %      \frac{dV_3(\theta)}{dt} &= \nabla V_3(\theta)^T\dot{\theta} \\
  %     &=(A^{T}C^{-1}(A\theta + b))^T(A^{T}C^{-1}A)^T(-A^{T}C^{-1}(A\theta + b))\\
    %    &= (A^{T}C^{-1}(A\theta + b))^T(-A^{T}C^{-1}A)^T(A^{T}C^{-1}(A\theta + b))
    %\end{split}
%\end{equation*}
%As before, notice that $-A^{T}C^{-1}A$ is negative definite. Therefore,
%\begin{equation*}
 %   \begin{split}
  %      \frac{dV_3(\theta)}{dt} & < 0 \mbox{ if } \theta \neq -A^{-1}b\\
  %     & = 0 \mbox{ if } \theta = -A^{-1}b
%    \end{split}
%\end{equation*}
%The function $V_3:\mathbb{R}^d\rightarrow\mathbb{R}$ serves as a Lyapunov function for the ODE given in \eqref{ode-h} and $\theta^* = -A^{-1}b$ is its unique globally asymptotically stable attractor.
This verifies (C6).

It follows from Theorem~\ref{MainResult} that both the iterates in the TDC($\lambda$)-schedule algorithm given by \eqref{TDCLambdaScehdule} remain uniformly bounded almost surely and moreover $(\theta_n,w_n) \rightarrow (\theta^*,\lambda(\theta^*)) = (-A^{-1}b, C^{-1}(A(-A^{-1}b) + b)) = (-A^{-1}b,0)$ almost surely. 
In particular, $\theta_n\rightarrow\theta^* = -A^{-1}b$ as $n\rightarrow\infty$ almost surely. 
The claim follows.
\end{proof}
\begin{theorem}
\label{TDC_app}
Under Assumptions 1-3 and 5, the TDC($\lambda$)-schedule algorithm given by \eqref{TDCLambdaScehdule} satisfies $\theta_{t}\rightarrow\theta^*=-A^{-1}b$ almost surely as $t\rightarrow\infty$.
\end{theorem}
Some remarks on the results and proofs are given in Appendix A4 that also contains empirical results obtained using the aforementioned algorithms. 

\section{Related Work and Conclusion}
Recent work by \cite{Online_off_policy}, \cite{SuttonBarto_book} and an earlier work by \cite{Dann} provide a comprehensive survey of TD based algorithms. However, for the sake of completeness we discuss some of the relevant works.
TD($\lambda$) with variable $\lambda$ presented in Chapter 12 of \cite{SuttonBarto_book} and \cite{PTDLambda} come close to our algorithms. 
However, the parameter $\lambda$ in those algorithms is a function of state. 
% Also, \cite{SuttonBarto_book} does not provide a particular way to choose the $\lambda$-function for different states.
Moreover, such a $\lambda$-function does not give arbitrary weights to different $n$-step returns.
In fact, to the best of our knowledge, no other variant of TD has looked into letting the user assign weights to different $n$-step returns.
% TD($\lambda$)-schedule on the other hand precisely describes how choosing a $\lambda$-schedule assigns weight to different $n$-step returns.
However, our $\lambda$-schedule procedure allows this by choosing appropriate $\lambda$ schedules.
State-dependent $\lambda$ can be derived as a special case of our $\lambda$-schedule procedure by letting $\lambda_{1}=\lambda(s_{t})$, $\lambda_{2}=\lambda(s_{t+1})$, etc. 

The convergence proofs presented in our paper differ from the proofs presented earlier in the TD($\lambda$)-literature and require much less verification since they
are based on the ODE method. 
Our two-timescale proof is novel in that such a proof under the Markov noise setting has not been presented before. 
Providing the proof of two-time scale iterates under Markov noise as here has been mentioned in \cite{Maei_PhD} as future work. 
We also mention that our proofs are presented under fairly general conditions and could be generalized further for a more general state-valued process.
These proofs also work for the case where the underlying Markov process does not possess a unique stationary distribution that can in turn also depend on the underlying parameters. See remarks in Appendix A6 for some further discussions on the proof techniques.
%  Algorithms presented here, at the cost of some additional memory and computation, offers this choice to the user. 

% Devising and comparing different $\lambda$-schedules is left for future work. 
% Another possible direction would be to extend the schedule-based algorithms to the control setting, for instance, through SARSA($\lambda$) (\cite{SuttonBarto_book}) or actor-critic methods (\cite{kondatsitsiklis, actor_critic}). 
Our work calls in for a comparative analysis of bias variance trade-off of all these variants of TD algorithm. Devising and comparing different $\lambda$-schedules is left for future work. 
Another possible direction would be to extend the schedule-based algorithms to the control setting, for instance, through SARSA($\lambda$) (\cite{SuttonBarto_book}) or actor-critic methods (\cite{kondatsitsiklis, actor_critic}). 
% Trying to see how many new lines can we fit? Just one more.
\bibliographystyle{plainnat}
\bibliography{TDWeights_References}
{\onecolumn
\aistatstitle{Appendix}
\section*{A1 \quad Coefficient of \texorpdfstring{$R_{t+1}$}{TEXT} and  \texorpdfstring{$V_{\theta}(s_{t+1})$}{TEXT}}
We had 
\begin{equation}
\begin{split}
    G_{t}^{\Lambda(\cdot)}(\theta) & = (1-\gamma)[\Lambda_{11}R_{t+1}] \\
 & + \gamma(1-\gamma)[\Lambda_{21}(R_{t+1}+V_{\theta}(s_{t+1})) + \Lambda_{22}(R_{t+1}+R_{t+2})] \\
 & + \gamma^2(1-\gamma)[\Lambda_{31}(R_{t+1} + V_{\theta}(s_{t+1})) + \Lambda_{32}(R_{t+1}+R_{t+2}+V_{\theta}(s_{t+2})) \\ & \quad \quad \quad \quad \quad +\Lambda_{33}(R_{t+1}+R_{t+2}+R_{t+3})] +\cdots \\
 \\ 
\end{split}
\end{equation}
Adding and subtracting $V_{\theta}(s_{t+1})$ to all the terms starting from $R_{t+1}+R_{t+2}$ in the above equation gives
\begin{equation}
\label{lambda_return_modified}
\begin{split}
    G_{t}^{\Lambda(\cdot)}(\theta) & = (1-\gamma)[\Lambda_{11}R_{t+1}] \\
 & + \gamma(1-\gamma)[\Lambda_{21}(R_{t+1}+V_{\theta}(s_{t+1})) + \Lambda_{22}(R_{t+1}+R_{t+2} + V_{\theta}(s_{t+1}) - V_{\theta}(s_{t+1}))] \\
 & + \gamma^2(1-\gamma)[\Lambda_{31}(R_{t+1} + V_{\theta}(s_{t+1})) + \Lambda_{32}(R_{t+1}+R_{t+2}+V_{\theta}(s_{t+2}) + V_{\theta}(s_{t+1}) - V_{\theta}(s_{t+1})) \\ & \quad \quad \quad \quad \quad +\Lambda_{33}(R_{t+1}+R_{t+2}+R_{t+3} + V_{\theta}(s_{t+1}) - V_{\theta}(s_{t+1}))] +\cdots \\
 \\ 
\end{split}
\end{equation}
The coefficient of $R_{t+1}$ is
\[= (1-\gamma)\Lambda_{11} + \gamma(1-\gamma)(\Lambda_{21} + \Lambda_{22}) + \gamma^2(1-\gamma)(\Lambda_{31} + \Lambda_{32} + \Lambda_{33}) + \cdots\]
\begin{equation*}
    \begin{split}
        & = (1-\gamma) \left[1 + \gamma + \gamma^2 + \cdots\right]\\
        & = (1-\gamma) \frac{1}{1-\gamma}\\
        & = 1.
    \end{split}
\end{equation*}
% The second equality follows from the fact that the matrix $\Lambda$ is row stochastic. 

Similarly, the coefficient of \(V_{\theta}(s_{t+1})\) is
\begin{equation*}
    \begin{split}
        & = \gamma(1-\gamma)(\Lambda_{21} + \Lambda_{22}) + \gamma^{2}(1-\gamma)(\Lambda_{31} + \Lambda_{32} + \Lambda_{33}) + \cdots\\
        & = \gamma(1-\gamma) \left[1 + \gamma + \gamma^2 + \cdots\right]\\
        & = \gamma(1-\gamma) \frac{1}{1-\gamma}\\
        & = \gamma.
    \end{split}
\end{equation*}
Equation \eqref{lambda_return_modified} then reduces to
\begin{equation}
\label{lambda_return_final}
\begin{split}
    G_{t}^{\Lambda(\cdot)}(\theta) & = R_{t+1} + \gamma V_{\theta}(s_{t+1}) \\
    & + \gamma\Lambda_{22}\{ (1-\gamma)(R_{t+2}) + \gamma(1-\gamma)\Big[\frac{\Lambda_{32}}{\Lambda_{22}}(R_{t+2}+V_{\theta}(s_{t+2})) + \frac{\Lambda_{33}}{\Lambda_{22}}(R_{t+2}+R_{t+3})\Big] \\
    & + \ldots - V_{\theta}(s_{t+1})  \}. \\
\end{split}
\end{equation}
\section*{A2 \quad Recursive equation for the \texorpdfstring{$\lambda$}{TEXT} 
     --schedule return}
The weight matrix was set to 
\begin{center}
     
    \[\Lambda = \begin{bmatrix}
        1 & 0 & 0 & 0 & 0 &0\ldots\\
        (1-\lambda_{1}) & \lambda_{1} & 0 & 0 & 0 &0\ldots \\
        (1-\lambda_{1}) & \lambda_{1}(1-\lambda_{2}) & \lambda_{1}\lambda_{2} & 0 & 0 &0\ldots \\
        (1-\lambda_{1}) & \lambda_{1}(1-\lambda_{2}) & \lambda_{1}\lambda_{2}(1-\lambda_{3}) & \lambda_{1}\lambda_{2}\lambda_{3} & 0 &0\ldots \\
        
        \vdots & \vdots & \vdots& \vdots &\ddots &\vdots
    \end{bmatrix}.\]
 \end{center}
Then \eqref{lambda_return_final} becomes 
\begin{equation*}
\begin{split}
    G_{t}^{[\lambda_1:]}(\theta) - V_{\theta}(s_t) & = R_{t+1} + \gamma V_{\theta}(s_{t+1}) - V_{\theta}(s_{t}) \\
    & + \gamma\lambda_1\{(1-\gamma)(R_{t+2}) + \gamma(1-\gamma)[(1-\lambda_{2})(R_{t+2}+V_{\theta}(s_{t+2})) + \lambda_{2}(R_{t+2}+R_{t+3})] \\
    & + \cdots - V_{\theta}(s_{t+1})\}. 
    \end{split}
\end{equation*}
We consider $\lambda$-schedules where $\exists L > 0$ such that $\lambda_j = 0$ for all $j > L$ and denote by $G_{t}^{[\lambda_{j}:\lambda_{L}]}$ the $\lambda$-schedule return at time $t$ following the schedule from $\lambda_{j}$ to $\lambda_{L}$. Therefore \eqref{lambda_return_final} becomes
\begin{equation*}
    \begin{split}
        G_{t}^{[\lambda_1:\lambda_L]}(\theta) & = R_{t+1} + \gamma V_{\theta}(s_{t+1}) \\
        & + \gamma\lambda_1\{(1-\gamma)(R_{t+2}) \\
        & \qquad + \gamma(1-\gamma)[(1-\lambda_{2})(R_{t+2}+V_{\theta}(s_{t+2})) + \lambda_{2}(R_{t+2}+R_{t+3})] \\
        & \qquad + \gamma^2(1-\gamma)[(1-\lambda_{2})(R_{t+2} + V_{\theta}(s_{t+2})) + \lambda_{2}(1-\lambda_{3})(R_{t+2} + R_{t+3} + V_{\theta}(s_{t+3})) \\
        & \qquad\qquad\qquad\qquad +\lambda_{2}\lambda_{3}(R_{t+2} + R_{t+3} + R_{t+3})] + \cdots\\
        & \qquad+\gamma^{L-1}(1-\gamma)[(1-\lambda_{2})(R_{t+2} + V_{\theta}(s_{t+2})) + \lambda_{2}(1-\lambda_{3})(R_{t+2} + R_{t+3} + V_{\theta}(s_{t+3})) \\&\qquad\qquad\qquad\qquad + \cdots + \lambda_{2}\lambda_{3}\ldots\lambda_{L}(R_{t+2} + R_{t+3} + \cdots + R_{t+L})] \\ &\qquad - V_{\theta}(s_{t+1})\}.
    \end{split}
\end{equation*}
\begin{equation*}
    \begin{split}
          G_{t}^{[\lambda_1:\lambda_L]}(\theta)& = R_{t+1} + \gamma(1-\lambda_{1})V_{\theta}(s_{t+1})  \\ 
           & + \gamma\lambda_1\{(1-\gamma)(R_{t+2}) \\
        & \qquad + \gamma(1-\gamma)[(1-\lambda_{2})(R_{t+2}+V_{\theta}(s_{t+2})) + \lambda_{2}(R_{t+2}+R_{t+3})] \\
        & \qquad + \gamma^2(1-\gamma)[(1-\lambda_{2})(R_{t+2} + V_{\theta}(s_{t+2})) + \lambda_{2}(1-\lambda_{3})(R_{t+2} + R_{t+3} + V_{\theta}(s_{t+3})) \\
        & \qquad\qquad\qquad\qquad +\lambda_{2}\lambda_{3}(R_{t+2} + R_{t+3} + R_{t+3})] + \cdots\\
        & \qquad+\gamma^{L-1}(1-\gamma)[(1-\lambda_{2})(R_{t+2} + V_{\theta}(s_{t+2})) + \lambda_{2}(1-\lambda_{3})(R_{t+2} + R_{t+3} + V_{\theta}(s_{t+3})) \\&\qquad\qquad\qquad\qquad + \cdots + \lambda_{2}\lambda_{3}\ldots\lambda_{L}(R_{t+2} + R_{t+3} + \cdots + R_{t+L})]\}\\
          & = R_{t+1} + \gamma(1-\lambda_{1})V_{\theta}(s_{t+1}) + \gamma\lambda_{1}G_{t+1}^{[\lambda_{2}:\lambda_{L}]}(\theta)\\
          & = R_{t+1} + \gamma \left[(1-\lambda_1)V_{\theta}(s_{t+1})+\lambda_1 G_{t+1}^{[\lambda_{2}:\lambda_{L}]}(\theta)\right].
    \end{split}
\end{equation*}
\section*{A3 \quad Gradient Based Algorithms}
\subsection*{A3.1 \quad Proof of Lemma 3.1}
\begin{customlemma}{3.1}\label{lemma3.1}
The objective function $J(\theta) = ||V_{\theta} - \Pi T^{\pi,[\lambda_1:\lambda_L]}V_{\theta}||_{D}^{2}$ can be equivalently written as
    $J(\theta) = \left(P_{\mu}^{\pi}\delta_{t}^{[\lambda_1:\lambda_L]}(\theta)\phi_{t}\right)^T \mathbb{E}_{\mu}\left[\phi_{t}\phi_{t}^T\right]^{-1}\left(P_{\mu}^{\pi}\delta_{t}^{[\lambda_1:\lambda_L]}(\theta)\phi_{t}\right).$\\
\end{customlemma}
\begin{proof}
The proof follows along the lines of \citet{Maei_PhD} and is reproduced here for completeness. We have
\begin{equation*}
G_{t}^{[\lambda_1:\lambda_L]}(\theta) = R_{t+1} + \gamma \left[(1-\lambda_1)V_{\theta}(s_{t+1})+\lambda_1 G_{t+1}^{[\lambda_{2}:\lambda_{L}]}(\theta)\right].
\end{equation*}
Next, we have the following:
\[\mathbb{E}_{\pi}\left[\delta_{t}^{[\lambda_{1}:\lambda_{L}]}(\theta)|s_{t} = s\right] = (T^{\pi[\lambda_{1}:\lambda_{L}]}V_{\theta}-V_{\theta})(s),\]
\[\mathbb{E}_{\mu}\left[\phi_{t}\phi_{t}^T\right] = \mathbb{E}_{\mu}\left[\sum_{s}d^\mu(s)\phi_{s}\phi_{s}^T\right] = \Phi^TD_{\mu}\Phi,\]
\begin{equation*}
    \begin{split}
        P_{\mu}^{\pi}\delta_{t}^{[\lambda_1:\lambda_L]}(\theta)\phi_{t} &= \sum_{s}d^{\mu}(s)\mathbb{E}_{\pi}\left[\delta_{t}^{[\lambda_{1}:\lambda_{L}]}(\theta)|s_{t} = s\right]\\
        &= \sum_{s}d^{\mu}(s)\left[(T^{\pi[\lambda_{1}:\lambda_{L}]}V_{\theta}-V_{\theta})(s)\right]\phi_{s}\\
        &= \Phi^TD_{\mu}(T^{\pi[\lambda_{1}:\lambda_{L}]}V_{\theta}-V_{\theta}).
    \end{split}
\end{equation*}
Now, the objective function can be written as (see \citet{Maei_PhD})
\begin{equation*}
    \begin{split}
        J(\theta) &= ||V_{\theta} - \Pi T^{\pi,[\lambda_1:\lambda_L]}V_{\theta}||_{D_{\mu}}^{2}\\
        &= \left(\Phi^TD_{\mu}(T^{\pi[\lambda_{1}:\lambda_{L}]}V_{\theta}-V_{\theta})\right)^T(\Phi_{t}^TD_{\mu}\Phi_{t})^{-1} \left(\Phi^TD_{\mu}(T^{\pi[\lambda_{1}:\lambda_{L}]}V_{\theta}-V_{\theta})\right)\\
        &=\left(P_{\mu}^{\pi}\delta_{t}^{[\lambda_1:\lambda_L]}(\theta)\phi_{t}\right)^T \mathbb{E}_{\mu}\left[\Phi_{t}^T\Phi_{t}\right]\left(P_{\mu}^{\pi}\delta_{t}^{[\lambda_1:\lambda_L]}(\theta)\phi_{t}\right).
    \end{split}
\end{equation*}
The claim follows.
\end{proof}
\subsection*{A3.2 \quad Proof of Theorem 3.2}
\begin{customthm}{3.2}\label{thm3.2}
$P_{\mu}^{\pi}\delta_{t}^{[\lambda_1:\lambda_L]}(\theta)\phi_{t} = \mathbb{E}_{\mu} \left[\delta_{t}^{[\lambda_{1}:\lambda_{L}],\rho}(\theta)\phi_{t}\right].$
\end{customthm}

\begin{proof}
We begin by expressing the RHS as:
\begin{equation*}
    \begin{split}
        \mathbb{E}_{\mu}\left[\delta_{t}^{[\lambda_{1}:\lambda_{L}],\rho}(\theta)\phi_{t}\right] & = \mathbb{E}_{s\sim d^{\mu}}\left[\mathbb{E}\left[\delta_{t}^{[\lambda_1:\lambda_L],\rho}(\theta)\phi_{t}|S_t = s\right]\right] \\
         &= \mathbb{E}_{s \sim d^{\mu}} \left[\mathbb{E}_{\mu}[\rho_{t:t}\delta_t(\theta)\phi_t|S_t=s\right] + \mathbb{E}_{\mu}\left[\rho_{t:t+1}\gamma\lambda_1\delta_{t+1}(\theta)\phi_{t}|S_t =s\right] + \cdots \\ 
         & + \mathbb{E}_{\mu}\left[\rho_{t:t+L}\gamma^L\lambda_1\ldots\lambda_L\delta_{t+L}(\theta)\phi_t|S_t=s\right]] \\
         &= \sum_{s} d^{\mu}(s) \mathbb{E}_{\pi}\left[(\delta_t(\theta) + \gamma\lambda_1\delta_{t+1}(\theta) + \gamma^2\lambda_1\lambda_2\delta_{t+2}(\theta) + \ldots + \gamma^L\lambda_1\ldots\lambda_L\delta_{t+L}(\theta))\phi_{t}|S_{t} = s,\pi\right] \\
        &= \sum_{s} d^{\mu}(s) \mathbb{E}_{\pi}\left[\delta_{t}^{[\lambda_1:\lambda_{L}]}(\theta)\phi_t|S_t = s,\pi\right]\\
        &= P_{\mu}^{\pi}\delta_{t}^{[\lambda_1:\lambda_L]}(\theta)\phi_t.
    \end{split}
\end{equation*}
The claim follows.
\end{proof}
\subsection*{A3.3 \quad Proof of Lemma 3.3}
\begin{customthm}{3.3}\label{thm3.3}
$\mathbb{E}_{\mu}\left[\rho_{t}\gamma\lambda_{1}\delta_{t+1}^{[\lambda_{2}:\lambda_{L}],\rho}(\theta)\phi_{t}\right] = \mathbb{E}_{\mu}\left[\rho_{t-1}\gamma\lambda_{1}\delta_{t}^{[\lambda_{2}:\lambda_{L}],\rho}(\theta)\phi_{t-1}\right]$.
\end{customthm}
    
\begin{proof}
The LHS, $\mathbb{E}_{\mu}\left[\rho_{t}\gamma\lambda_{1}\delta_{t+1}^{[\lambda_{2}:\lambda_{L}],\rho}(\theta)\phi_{t}\right]$
\[=\mathbb{E}_{\mu}\left[\rho_{t}\gamma\lambda_{1}\left(\rho_{t+1}\delta_{t+1}\phi_{t} + \gamma\lambda_{2}\rho_{t+1}\rho_{t+2}\delta_{t+2}\phi_{t}+\ldots+\gamma^{L-1}\lambda_{2}\ldots\lambda_{L}\rho_{t+1}\ldots\rho_{t+L}\delta_{t+L}\phi_{t}\right)\right]] \]
\[= \gamma\lambda_{1}\left(\mathbb{E}_{\mu}[\rho_{t:t+1}\delta_{t+1}\phi_{t}] + \gamma\lambda_{2}\mathbb{E}_{\mu}[\rho_{t:t+2}\delta_{t+2}\phi_{t}] + \ldots + \gamma^{L-1}\lambda_{2}\ldots\lambda_{L}\mathbb{E}_{\mu}[\rho_{t:t+L}\delta_{t+L}\phi_{t}]\right)\]

where $\rho_{i:j} \triangleq \rho_{i}\rho_{i+1}\ldots\rho_{j}$. Under stationarity, $\mathbb{E}_{\mu}\left[\rho_{t:t+j}\delta_{t+j}\phi_{t}\right] = \mathbb{E}_{\mu}\left[\rho_{t-1:t+j-1}\delta_{t+j-1}\phi_{t-1}\right]$, $\forall j\geq0$. 
Therefore, the LHS reduces to
\begin{equation*}
    \begin{split}
        \gamma\lambda_{1}(& \mathbb{E}_{\mu}[\rho_{t-1:t}\delta_{t}\phi_{t-1}] + \gamma\lambda_{2}\mathbb{E}_{\mu}[\rho_{t-1:t+1}\delta_{t+1}\phi_{t-1}] + \cdots \\
        & + \gamma^{L-1}\lambda_{2}\ldots\lambda_{L}\mathbb{E}_{\mu}[\rho_{t-1:t+L-1}\delta_{t+L-1}\phi_{t-1}])
    \end{split}
\end{equation*}
\begin{equation*}
\begin{split}
    &=\mathbb{E}_{\mu}\left[\rho_{t-1}\gamma\lambda_{1}\left(\rho_{t:t}\delta_{t} + \gamma\lambda_{2}\rho_{t:t+1} + \ldots \gamma^{L-1}\lambda_{2}\ldots\lambda_{L}\delta_{t:t+L-1}\delta_{t+L-1}\right)\phi_{t-1}\right]\\
    &= \mathbb{E}_{\mu}\left[\rho_{t-1}\gamma\lambda_{1}\delta_{t}^{[\lambda_{2}:\lambda_{L}],\rho}(\theta)\phi_{t-1}\right].
\end{split}
\end{equation*}
The claim follows.
\end{proof}

\subsection*{A3.4 \quad Proof of Theorem 3.4}
\begin{customthm}{3.4}\label{thm3.4}
Define the eligibility vector $z_{t} = \sum_{i = t-L}^{t}\left[\rho_{t}\left(\Pi_{j=1}^{t-i}\rho_{t-j}\gamma\lambda_j\right)\phi_i\right].$ 
Then,
$\mathbb{E}_{\mu}\left[\delta_{t}^{[\lambda_{1}:\lambda_{L}],\rho}(\theta)\phi_{t}\right] = \mathbb{E}_{\mu}\left[\delta_{t}(\theta)z_{t}\right]$.
\end{customthm}

\begin{proof}
We have
\begin{equation*}
\begin{split}
G_{t}^{[\lambda_1:\lambda_L],\rho}(\theta) & = \rho_t\left(R_{t+1} + \gamma \left[(1-\lambda_1)V_{\theta}(s_{t+1})+\lambda_1 G_{t+1}^{[\lambda_{2}:\lambda_{L}],\rho}(\theta)\right]\right)
 \\
 &= \rho_{t}\left(R_{t+1} + \gamma\theta^T\phi_{t+1} - \theta^T\phi_t + \theta^T\phi_t\right) - \rho_t\gamma\lambda_1(\theta^T\phi_{t+1}) + \rho_{t}\gamma\lambda_1G_{t+1}^{[\lambda_2:\lambda_L],\rho}(\theta)\\ 
 & = \rho_t\delta_{t}(\theta) + \rho_t\theta^T\phi_t + \rho_{t}\gamma\lambda_1\left(G_{t+1}^{[\lambda_2:\lambda_L],\rho}(\theta) - \theta^T\phi_{t+1}\right)\\
&= \rho_t\delta_{t}(\theta) + \rho_t\theta^T\phi_t + \rho_{t}\gamma\lambda_1\delta_{t+1}^{[\lambda_2:\lambda_L],\rho}(\theta).
\end{split}
\end{equation*}

From the earlier definition we have, 
\begin{equation*}
\begin{split}
\delta_{t}^{[\lambda_1:\lambda_L],\rho}(\theta) & = G_{t}^{[\lambda_1:\lambda_L],\rho}(\theta) - \theta^T\phi_{t}\\
&= \rho_t\delta_{t}(\theta) + (\rho_t-1)\theta^T\phi_t + \rho_{t}\gamma\lambda_1\delta_{t+1}^{[\lambda_2:\lambda_L],\rho}(\theta).
\end{split}
\end{equation*}

Taking expectation on both sides, using Lemma 3.3 and using the identity $\mathbb{E}[(\rho_{t}-1)\theta^{T}\phi_{t}\phi_{k}] = 0$  $\forall k\leq t$,
\begin{equation*}
\begin{split}
\mathbb{E}_{\mu}\left[\delta_{t}^{[\lambda_1:\lambda_L],\rho}(\theta)\phi_t\right] 
&= \mathbb{E}_{\mu}\left[\rho_t\delta_t(\theta)\phi_t\right] + \mathbb{E}_{\mu}\left[\rho_t\gamma\lambda_1\delta_{t+1}^{[\lambda_2:\lambda_L],\rho}(\theta)\phi_t\right]\\
&= \mathbb{E}_{\mu}\left[\rho_t\delta_t(\theta)\phi_t\right] + \mathbb{E}_{\mu}\left[\rho_{t-1}\gamma\lambda_1\delta_{t}^{[\lambda_2:\lambda_L],\rho}(\theta)\phi_{t-1}\right]\\
&= \mathbb{E}\left[\rho_t\delta_t(\theta)\phi_t + \rho_{t-1}\gamma\lambda_1\left((\rho_t\delta_t(\theta)+\rho_t\gamma\lambda_2\delta_{t+1}^{[\lambda_3:\lambda_L],\rho}(\theta))\phi_{t-1}\right)\right]\\
&= \mathbb{E}\left[\rho_t\delta_t(\theta)\phi_t + \rho_{t-1}\rho_{t}\gamma\lambda_1\delta_t(\theta)\phi_{t-1} + \rho_{t}\rho_{t-1}\gamma^2\lambda_1\lambda_2 \delta_{t+1}^{[\lambda_3:\lambda_L],\rho}(\theta))\phi_{t-1} \right]\\
&= \mathbb{E}_{\mu}[\delta_{t}(\theta)(\rho_{t}\phi_{t} + \rho_{t}\rho_{t-1}\gamma\lambda_{1}\phi_{t-1} + \rho_{t}\rho_{t-1}\rho_{t-2}\gamma^2\lambda_{1}\lambda_{2}\phi_{t-2} + \cdots \\ & \qquad\quad + \rho_{t}\rho_{t-1}\ldots\rho_{t-L}\gamma^{L}\lambda_{1}\lambda_{2}\ldots\lambda_{L}\phi_{t-L})]\\
&= \mathbb{E}_{\mu}\left[\delta_{t}(\theta)z_{t}\right]].
\end{split}
\end{equation*}
The claim follows.
\end{proof}

\section*{A4 \quad Convergence Analysis}
\subsection*{A4.1 \quad TD\texorpdfstring{($\lambda$)}{TEXT}-schedule}
\begin{customthm}{A4.1}
 \label{lemma7_tsitsiklis}
 $\mathbb{E}_{0}[\phi(s_0)\phi(s_m)^T] = \Phi^TDP^m\Phi$
\end{customthm}
 \begin{proof}
 Refer to Lemma 7 in \citet{TsitsiklisVanRoy}.
\end{proof}
\textbf{Proposition 4.1.}
    \textit{The matrix $A$ is negative definite.}
\begin{proof}
Let $||\cdot||_D$ be the weighted quadratic norm defined by $||V||_D = \sqrt{V^TDV} = \sqrt{\sum_{s=1}^{n}d^{\pi}(s)V(s)^{2}}$.  
Since the expectation is with respect to the steady state distribution of $\{X_{t}\}$, the matrix $A = \mathbb{E}_{0}[A(X_{t})]$ can be equivalently written as $A = \mathbb{E}_{0}[\gamma z_{0}\phi_{1}^{T} - z_{0}\phi_{0}^{T}]$.
Using Lemma \ref{lemma7_tsitsiklis}, the second term in $A$ can be written as
\begin{equation}
\begin{split}
 \mathbb{E}_{0}[z_{0}\phi(s_{0})^T]
  & = \mathbb{E}_{0}\left[\sum_{k=-L}^{0}\left(\prod_{j=1}^{-k}\gamma\lambda_j\right)\phi(s_k)\phi(s_0)^T\right]
 = \mathbb{E}_{0}\left[\sum_{k=0}^{L}\left(\prod_{j=1}^{k}\gamma\lambda_j\right)\phi(s_{-k})\phi(s_0)^T\right]\\
 &= \Phi^T D\left[\sum_{k=0}^{L}\left(\prod_{j=1}^{k}\gamma\lambda_j\right) P^{k}\right]\Phi.
\end{split}
\end{equation}
Similarly, the first term in $A$ can be written as:
\begin{eqnarray*}
\mathbb{E}_{0}[z_{0}\phi(s_{1})^T]
  &=& \mathbb{E}_{0}\left[\sum_{k=-L}^{0}\left(\prod_{j=1}^{-k}\gamma\lambda_j\right)\phi(s_k)\phi(s_1)^T\right]
 = \mathbb{E}_{0}\left[\sum_{k=0}^{L}\left(\prod_{j=1}^{k}\gamma\lambda_j\right)\phi(s_{-k})\phi(s_1)^T\right]\\
 &=& \Phi^T D\left[\sum_{k=0}^{L}\left(\prod_{j=1}^{k}\gamma\lambda_j\right) P^{k+1}\right]\Phi.
\end{eqnarray*}
Therefore, 

\[ A = \mathbb{E}_0 [z_0 (\gamma\phi_1 - \phi_0)^T] = \Phi^T D\left[\sum_{k=0}^{L}\left(\prod_{j=1}^{k}\gamma\lambda_j\right) (\gamma P^{k+1} - P^{k})\right]\Phi. \] 

% \[A = \Phi^T D\left[\sum_{k=0}^{L}\left(\prod_{j=1}^{k}\gamma\lambda_j\right) (\gamma P^{k+1} - P^{k})\right]\Phi\]
Now, the matrix $\sum_{k=0}^{L}\left(\prod_{j=1}^{k}\gamma\lambda_j\right) (\gamma P^{k+1} - P^{k})$
\begin{equation*}
\begin{split}
& = \gamma P - I +\gamma^2\lambda_{1}P^2 - \gamma\lambda_{1}P + \gamma^3\lambda_1\lambda_2P^3 - \gamma^2 \lambda_1\lambda_2P^2 + \ldots \\ &+\gamma^{L}\lambda_1\lambda_2\ldots\lambda_{L-1}P^{L} - \gamma^{L-1}\lambda_1\lambda_2\ldots\lambda_{L-1}P^{L-1}+ \gamma^{L+1}\lambda_{1}\lambda_{2}\ldots\lambda_{L}P^{L+1} - \gamma^{L}\lambda_{1}\lambda_{2}\ldots\lambda_{L}P^{L}\\
& = \gamma(1-\lambda_{1})P + \gamma^2\lambda_1(1-\lambda_2)P^2 + \ldots + \gamma^L\lambda_{1}\lambda_{2}\ldots\lambda_{L-1}(1-\lambda_{L})P^{L} + \gamma^{L+1}\lambda_{1}\lambda_{2}\ldots\lambda_{L}P^{L+1}-I.
\end{split}
\end{equation*}
%Denote $M$ = $\gamma(1-\lambda_{1})P + \gamma^2\lambda_1(1-\lambda_2)P^2 + \ldots + \gamma^L\lambda_{1}\lambda_{2}\ldots\lambda_{L-1}(1-\lambda_{L})P^{L} + \gamma^{L+1}\lambda_{1}\lambda_{2}\ldots\lambda_{L}P^{L+1}$. 
We define \[M = \gamma(1-\lambda_{1})P + \gamma^2\lambda_1(1-\lambda_2)P^2 + \ldots + \gamma^L\lambda_{1}\lambda_{2}\ldots\lambda_{L-1}(1-\lambda_{L})P^{L} + \gamma^{L+1}\lambda_{1}\lambda_{2}\ldots\lambda_{L}P^{L+1}.\] 
and thus $A = \Phi^TD(M-I)\Phi$. Now,
\begin{equation*}
\begin{split}
& ||MV||_{D} \leq [\gamma(1-\lambda_{1}) + \gamma^2\lambda_{1}(1-\lambda_2)+\ldots+\gamma^{L}\lambda_{1}\ldots\lambda_{L-1}(1-\lambda_{L}) +  \gamma^{L+1}\lambda_{1}\lambda_{2}\ldots\lambda_{L}]||V||_{D}\\
&\quad \quad \quad \leq [\gamma(1-\lambda_{1}) + \gamma\lambda_{1}(1-\lambda_2)+\ldots+\gamma\lambda_{1}\ldots\lambda_{L-1}(1-\lambda_{L}) +  \gamma\lambda_{1}\lambda_{2}\ldots\lambda_{L}]||V||_{D}\\
& \quad \quad \quad = \gamma[1-\lambda_1+\lambda_1-\lambda_1\lambda_2+\ldots - (\lambda_{1}\ldots\lambda_{L}) + (\lambda_{1}\ldots\lambda_{L})]||V||_{D}\\
&\quad \quad \quad = \gamma||V||_{D}.
\end{split}
\end{equation*}
% The first inequality above follows from the facts that $\parallel PV\parallel_D \leq \parallel V\parallel_D$, $\parallel P^{2} V\parallel_D =$ $\parallel P (PV)\parallel_D \leq \parallel PV\parallel_D$
% $\leq \parallel V\parallel_D$ etc. 
The first inequality above follows from the facts that $||PV||_D \leq ||V||_D$, $||P^{2} V||_D = ||P(PV)||_D \leq ||PV||_D
\leq ||V||_D$ etc. 
The second inequality follows from the fact that $\gamma^i \leq \gamma$ for all $i \in \mathbb{N}$ when $\gamma \leq 1$. 
Thereafter following similar arguments as in Lemma 6.6 in \citet{NDP_book}, $A$ can be shown to be negative definite.
\end{proof}

% Consider now a sequence $\{t(n)\}$ of time points defined as follows:
% $t(0)=0$, ${\displaystyle t(n) = \sum_{k=0}^{n-1} \alpha_k}$, $n\geq 1$.
% Now define the algorithm's trajectory $\bar{\theta}(t)$ according to: $\bar{\theta}(t(n)) = \theta_n$, $\forall n$, and
% with $\bar{\theta}(t)$ defined as a continuous linear interpolation on all intervals $[t(n),t(n+1)]$.

\subsection*{A4.2 \quad Convergence of GTD\texorpdfstring{($\lambda$)}{TEXT}-Schedule}
Recall that the iterates for GTD($\lambda$)-schedule are given by:
\begin{equation}
\label{GTDLambdaSchedule_app}
\begin{split}
    \theta_{t+1} = \theta_t + \alpha_t\left((\phi_t-\gamma\phi_{t+1})z_t^T w_t\right),\\
    w_{t+1} = w_t + \beta_t\left(\delta_{t}(\theta)z_{t}-\phi_t\phi_t^{T}w_{t} \right).
    \end{split}
\end{equation}
\begin{customthm}{4.4}\label{thm4.4}
Under Assumptions 1-4, $\{\theta_{t}\}$ in the GTD($\lambda$)-Schedule iterate given in equation \eqref{GTDLambdaSchedule_app} converges almost surely to $-A^{-1}b$.
\end{customthm}
\begin{proof}
We rewrite \eqref{GTDLambdaSchedule_app} as follows:
\begin{equation}
\label{GTDLambdaChangedIterate}
    \xi_{t+1} = \xi_{t} + \alpha_{t}\sqrt{\eta}(G_{t}\xi_{t} + g_{t+1}),
\end{equation}
where, $\xi_{t}^T = (w_{t}^{T}/\sqrt{\eta},\theta_{t}^{T})$, $g_{t+1}^T = (R_{t+1}z_{t}^{T},0^T)$ and
\begin{center}
    \[G_{t} = \begin{bmatrix}
        -\sqrt{\eta}\phi_{t}\phi_{t}^T & z_{t}(\gamma\phi_{t+1} - \phi_{t})^T\\
        -(\gamma\phi_{t+1} - \phi_{t})z_{t}^T & 0
    \end{bmatrix}.\]
 \end{center}
Equation \eqref{GTDLambdaChangedIterate} can be rewritten as
\begin{equation}
\label{GTDL}
    \xi_{t+1} = \xi_{t} + \alpha_{t}\sqrt{\eta}(G(X_{t})\xi_{t} + g(X_{t})),
\end{equation}
where,
\[ G(X_{t})=
\begin{bmatrix}
    -\sqrt{\eta}C(X_{t})  &  A(X_{t})      \\
    -A(X_{t})^{T}  &  0      
\end{bmatrix}
\mbox{ and } g(X_{t}) =
\begin{bmatrix}
    b(X_{t})     \\
    0      
\end{bmatrix}. 
\]
The matrix 
\[ G = \mathbb{E}_{0}[G(X_{t})]=
\begin{bmatrix}
    -\sqrt{\eta}C  &  A      \\
    -A^{T}  &  0      
\end{bmatrix}
\]
can be easily shown to be negative definite (see \citet{Maei_PhD}). 
Verification of Conditions (B1)-(B5) can now be done in a similar manner as in
Theorem 4.3 and it can be shown that $\xi_{t} \rightarrow -G^{-1}g$. Using the following formula for inverse of a block matrix
\[
\begin{bmatrix}
    A_{11} &  A_{12}      \\
    A_{21}  &  0      
\end{bmatrix}^{-1}
=
\begin{bmatrix}
    0 & A_{21}^{-1}    \\
    A_{12}^{-1} & -A_{12}^{-1}A_{11}A_{21}^{-1}      
\end{bmatrix}, 
\]
we have, 
\[
\begin{bmatrix}
    w_{t}/\sqrt{\eta} \\
    \theta_{t}      
\end{bmatrix}
\rightarrow
\begin{bmatrix}
    0 & (A^T)^{-1}   \\
    -A^{-1} & A^{-1}\sqrt{\eta}CA^{-1}      
\end{bmatrix} 
\begin{bmatrix}
    b\\
    0
\end{bmatrix}.
\]
Therefore $\theta_{t} \rightarrow -A^{-1}b$.
\end{proof}

\section*{A5 \quad Experiments}

\subsection*{A5.1 \quad 100-State Random Walk}
 This is a randomly generated discrete MDP with 100 states and 5 actions in each state. The transition probabilities are uniformly generated from $[0,1]$ with a small additive constant. The rewards are also uniformly generated from $[0,1]$. The policy and the start state distribution are also generated in a similar way and the discount factor $\gamma = 0.95$. See \cite{Dann} for a more detailed description. Tabular features are used in both cases. Figure \ref{100-state RW(a)} and \ref{100-state RW(b)} plot the results on this MDP.

\begin{figure}[H]
    \centering
    \subfloat[$\alpha = 0.005$]{
    \label{100-state RW(a)}
    {\includegraphics[width=0.5\linewidth]{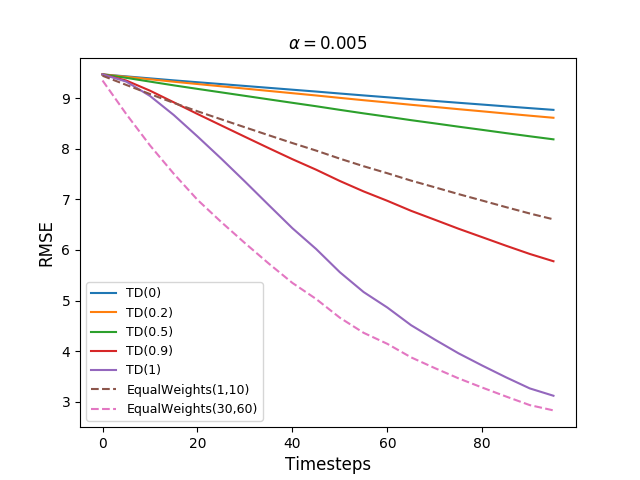}}
    }%
    %\quad
    \subfloat[$\alpha = 0.01$]{
    \label{100-state RW(b)}
    {\includegraphics[width=0.5\linewidth]{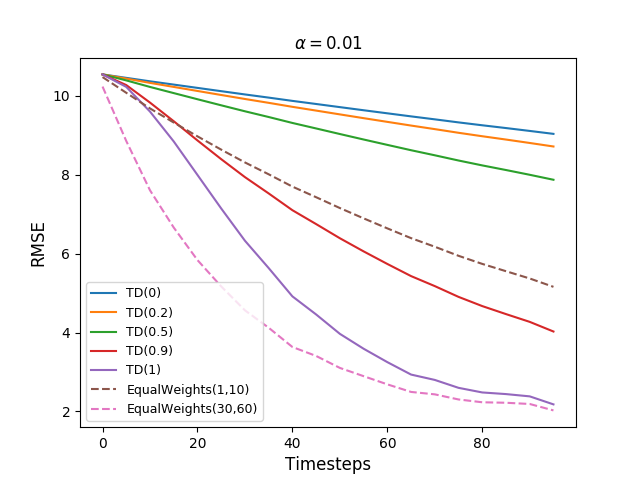}}
    }%
    \caption{RMSE across time-steps for 100-State Random Walk Example.}
    %\label{Baird_fig}%
\end{figure}

In both the cases it is observed that, for an episode length of 100, EqualWeights($30,60$) decreases the RMSE most rapidly. This enforces the idea that when the episode lengths are long, combining some intermediate $n$-step returns helps reduce the RMSE faster. A possible future direction to look at would be how to find the best choice of $n$-step returns that reduces the RMSE best.

\subsection*{A5.2 \quad Random Chain} 
The second example considers a simplified version of a linear random walk called \emph{Random Chain} (\citet{ajin}). 
It consists of 15 states arranged in a linear fashion with two additional absorbing states. 
At any state $s$, the agent can pick one of two actions from \{$L$, $R$\}. 
The action $L$ takes the agent to the state $s-1$ w.p. 0.9 and to the state $s+1$ w.p. 0.1, while the action $R$ takes the agent to the state $s-1$ with probability 0.1 and to the state $s+1$ w.p. 0.9.
The reward associated with all transitions is zero except all transitions to states 5 and 10, in which case the reward is 1
\begin{figure}[H]
    \centering
    {\includegraphics[width=1\linewidth]{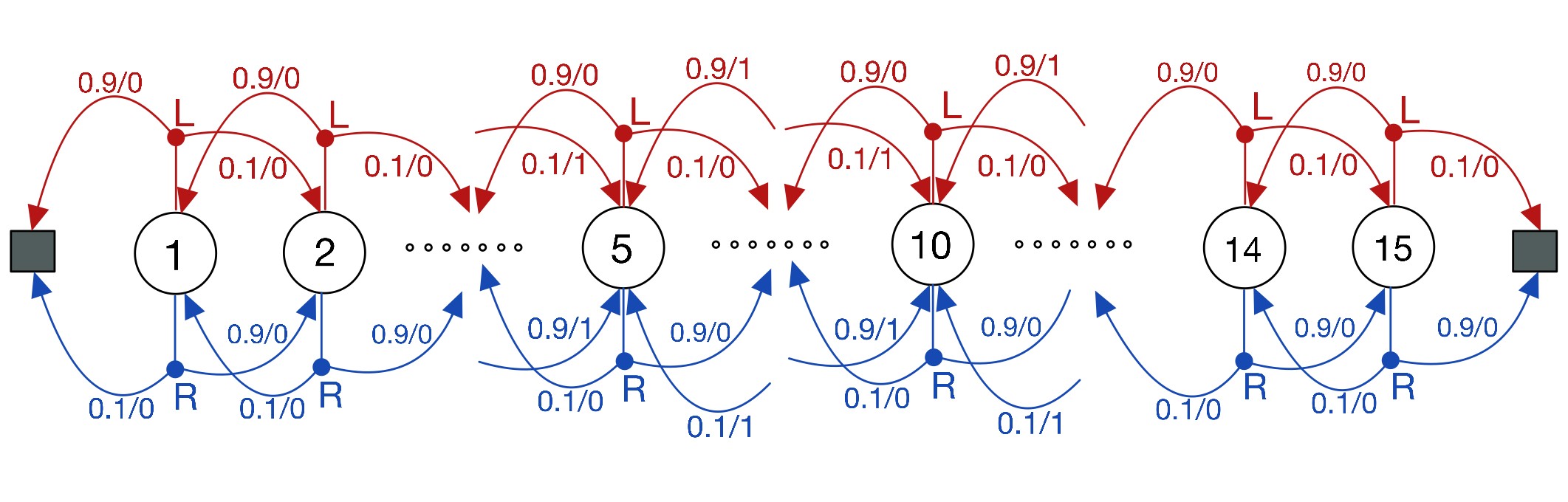}}
    %\quad
    \caption{The Random Chain setting from \citet{ajin}}
\end{figure}
The discount factor $\gamma = 0.9$ and the start state in each episode is chosen randomly. 
We compare our gradient algorithms GTD($\lambda$)-Schedule and TDC($\lambda$)-Schedule with the gradient algorithms GTD2 and TDC (\cite{FastGradient}).
We highlight that this example is in the off-policy setting, where the behaviour policy $\mu$ chooses both actions $L$ and $R$ w.p. 0.5 while the target policy $\pi$ chooses actions $L$ and $R$ w.p. 0.6 and 0.4 respectively. 
Figure \ref{Random_Chain} discusses the results.
Note that the algorithm GTD($\lambda$) proposed in \citet{Maei_PhD} (Chapter-6) does not converge on the Baird's example for high values of $\lambda$ (see Appendix A5.3). Hence, we do not compare GTD($\lambda$)-schedule and TDC($\lambda$)-schedule with GTD($\lambda$). 

\begin{figure}[H]
\label{Random_Chain_fig}
    \centering
    \subfloat[TDC vs TDC($\lambda$)-Schedule]{
    {\includegraphics[width=0.49\linewidth]{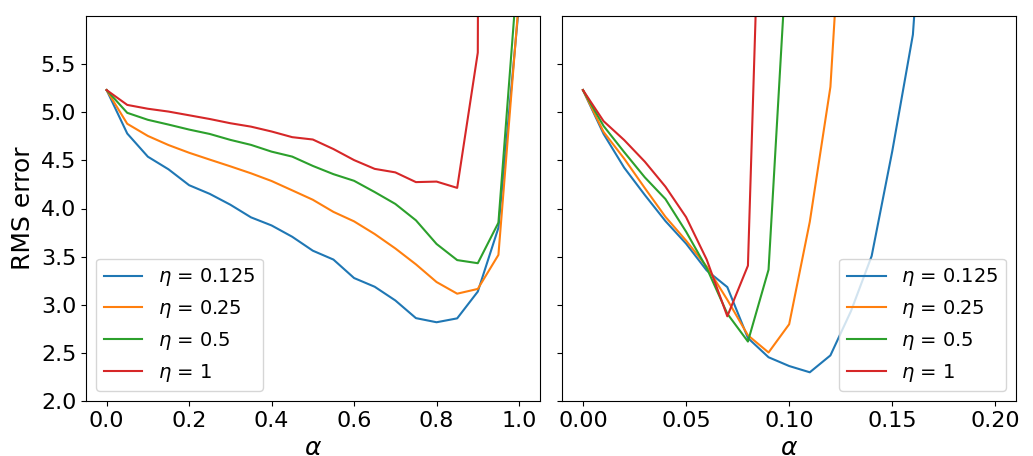}}
    % {\includegraphics[width=0.242\linewidth]{Random_Chain_TDC.png}}
    % {\includegraphics[width=0.242\linewidth]{Random_Chain_TDC_Schedule.png}}
    }%
    %\quad
    \subfloat[GTD2 vs GTD($\lambda$)-Schedule]{
    {\includegraphics[width=0.49\linewidth]{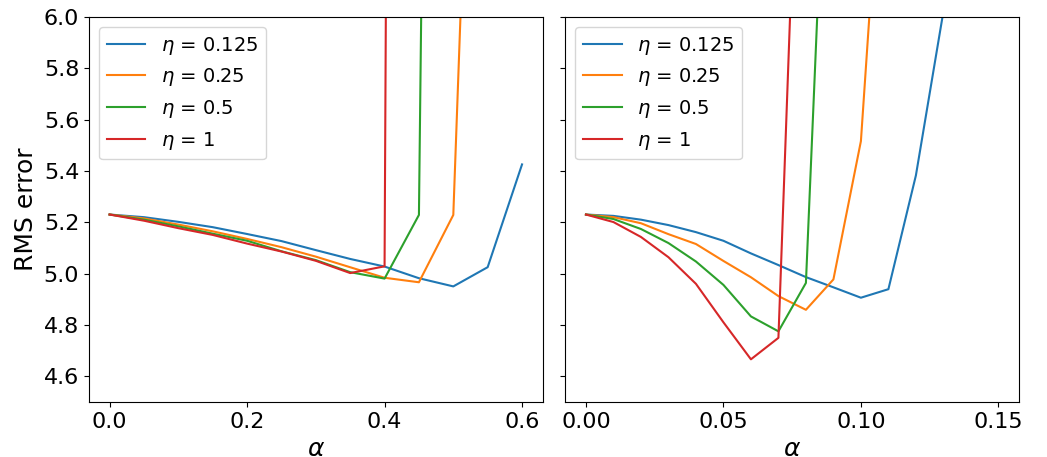}}
    % {\includegraphics[width=0.242\linewidth]{Random_Chain_GTD.png}}
    % {\includegraphics[width=0.242\linewidth]{Random_Chain_GTD_Schedule.png}}
    }%
    \caption{RMSE (averaged over 50 episodes and 50 runs) for different values of step-size $\alpha$ for the Random Chain example with Tabular features is plotted. (a) compares the error for TDC algorithm with various values of $\eta = \alpha/\beta$ against TDC($\lambda$)-schedule with \emph{EqualWeights}$(2,4)$. TDC($\lambda$)-schedule attains a lower RMSE than TDC for every corresponding $\eta$ and some $\alpha$. Similarly (b) compares the RMSE for the same problem (averaged over 50 episodes and 50 runs) for GTD2 and GTD($\lambda$)-Schedule and it is observed that GTD$(\lambda)$-Schedule again attains a lower RMSE for each corresponding value of $\eta$.}
    \label{Random_Chain}%
\end{figure}
\subsection*{A5.3 \quad Baird's Counter example}
The off-policy algorithms developed in this paper namely GTD($\lambda$)-schedule, TDC($\lambda$)-schedule as well as the off-policy analog of our TD($\lambda$)-schedule algorithm are applied on the Baird's counter example. Figure \ref{Baird_fig_a} plots the Root Mean Square Projected Bellman Error (RMSPBE) accross time-steps. \citet{Maei_PhD} presented the GTD($\lambda$) algorithm with the following iterates:
\begin{gather}
\label{gtd_theta}
    \theta_{t+1} = \theta_{t} + \alpha_{t}(\delta_{t}z_{t} - \gamma(1-\lambda)(z_{t}^{T}w_{t})\phi_{t+1})\\
    w_{t+1} = w_{t} + \beta_{t}(\delta_{t}z_{t} - (w_{t}^{T}\phi_{t})\phi_{t})
\end{gather}
Figure \ref{Baird_fig_b} shows that GTD($\lambda$) diverges for high values of $\lambda$ and therefore isn't compared with the $\lambda$-schedule algorithms. The divergence is evident from the fact that as $\lambda \rightarrow 1$, iterate \eqref{gtd_theta} tends towards the standard TD($\lambda$) algorithm, which is known to diverge in the off-policy case. 
\begin{figure}[H]
    \centering
    \subfloat[Schedule based Algorithms]{
    \label{Baird_fig_a}
    {\includegraphics[width=0.4\linewidth]{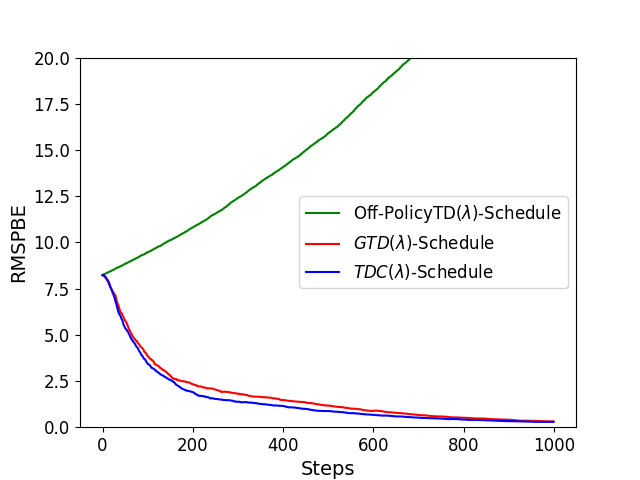}}
    }%
    %\quad
    \subfloat[GTD($\lambda$)]{
    \label{Baird_fig_b}
    {\includegraphics[width=0.4\linewidth]{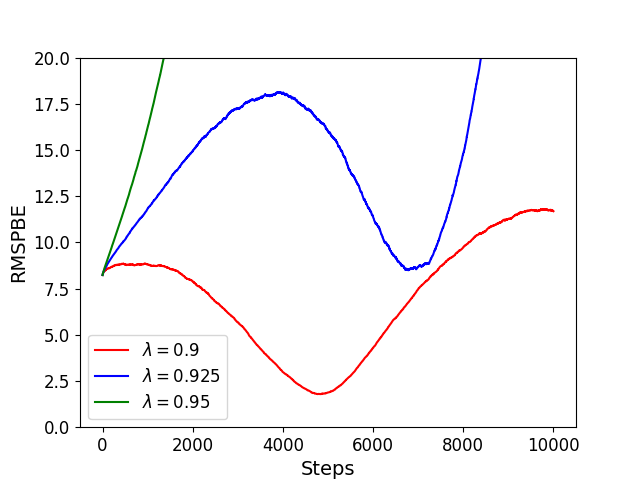}}
    }%
    \caption{RMSPBE across time-steps for Baird's Off-Policy Example. Plot (a) shows that Off-Policy TD($\lambda$)-Schedule diverges while both the gradient-based algorithms, GTD($\lambda$)-schedule and TDC($\lambda$)-schedule converge. The $\lambda$-Schedule chosen in all the three cases is EqualWeights(4,6). The step-size for Off-Policy TD($\lambda$) is $\alpha = 0.005$ and the step-sizes for both the gradient methods are $\alpha = 0.005$, $\beta = 0.05$, respectively. Plot (b) shows that GTD($\lambda$) (\citet{Maei_PhD}) does not converge for high values of $\lambda$.}
    %\label{Baird_fig}%
\end{figure}
\section*{A6 \quad Remarks on the proof techniques}
\begin{itemize}
    \item \textbf{Remark 1.} The proof of Theorem~4.3 was shown using Assumptions 1--3 and crucially relied on the matrix $A$ being negative definite. Assumption 2 also required that the process $\{s_t\}$ possesses a unique stationary distribution. The analysis of stability and convergence of stochastic recursions driven by Markov noise given in \citet{Arunselvan1} and \citet{BorkarBook} allows for multiple stationary distributions. Further, the set of states can be a compact metric space or a complete and separable metric space such as $\mathbb{R}^m$ but under an additional requirement. Either of these conditions would ensure that the set of time-dependent state-distribution probability measures remains tight and so have limit points in $D(\theta)$. In our case, as with the proof of TD($\lambda$), see \citet{NDP_book}, we let $S$ and hence $\check{S}$ to be finite sets. This provides improved clarity, for instance, Proposition~4.1 could be shown and we are able to show that the algorithm converges to the fix-point of the TD($\lambda$)-Schedule algorithm. Nonetheless the analysis under more general assumptions such as on the state process can be carried out using the results of \citet{Arunselvan1} and \citet{BorkarBook}. In fact an analysis of TD(0) under general requirements has been shown in \citet{Arunselvan1}. A similar analysis can also be carried out for the TD($\lambda$)-schedule algorithm.
    
    \item \textbf{Remark 2.} Regarding proof of Theorem 4.5, Theorem 10 of \citet{chandru-SB} shows that both the iterates remain stable almost surely under similar assumptions as (C1)-(C4). The analysis there is however carried out for the case when the noise sequence is a martingale difference. The same will also work for the case of Markov noise following similar arguments as in \citet{Arunselvan1} and \citet{BorkarBook}. Further, in Theorem 2.6 of \citet{prasenjit-SB}, the convergence of two-timescale stochastic approximation under Markov noise is shown assuming that both the iterates remain stable, i.e., under Assumptions (C1), (C2), (C5), (C6) and under the additional requirement of iterate stability. Assumptions (C3)-(C4) in addition to (C1)-(C2) are sufficient requirements to show the stability of the iterates (see \citet{chandru-SB}). Thus the conditions (C1)-(C6) are sufficient to show both the stability and convergence of two-timescale stochastic approximation iterates.
\end{itemize}
% \bibliographystyle{plainnat}
% \bibliography{TDWeights_References}
\vfill
}

\end{document}